\documentclass[twoside,11pt]{article}
\usepackage[preprint]{jmlr2e}

\usepackage{graphics}
\usepackage{caption}
\usepackage{subcaption}
\usepackage[ruled,vlined]{algorithm2e}
\usepackage[utf8]{inputenc}
\usepackage{amsmath}
\usepackage{mathtools}
\usepackage{xcolor}
\usepackage{verbatim}

\usepackage{indentfirst}

\usepackage{paralist}
\usepackage[T1]{fontenc}
\usepackage{pstool}
\usepackage{todonotes}

\RequirePackage{algorithm}
\RequirePackage{algorithmic}
\newcommand{\ep}{\epsilon}
\DeclareMathOperator{\Ker}{Ker}

\DeclareMathOperator{\PCech}{\mathcal{P}_0^{\breve{C}ech}}
\DeclareMathOperator{\PCechq}{\mathcal{P}_{\textit{q}}^{\breve{C}ech}}

\DeclareMathOperator{\PDTMq}{\mathcal{P}_{\textit{q}}^{DTM_{\textit{p}}}}
\DeclareMathOperator{\PDTMoneq}{\mathcal{P}_{\textit{q}}^{DTM_1}}
\DeclareMathOperator{\PDTMone}{\mathcal{P}_0^{DTM_1}}

\DeclareMathOperator{\Image}{Im}

\newcommand{\EE}{\mathbb{E}}

\DeclareMathOperator{\db}{d_B}
\DeclareMathOperator{\diam}{diam}

\usepackage[symbol]{footmisc}

\hypersetup{colorlinks=true,linkcolor=blue,anchorcolor=blue,citecolor=blue,filecolor=blue, urlcolor=blue,bookmarksnumbered=true,pdfstartview=FitB}

\usepackage{lastpage}
\jmlrheading{}{}{}{}{}{}{}

\ShortHeadings{Differentially Private Topological Data Analysis}{Kang, Kim, Sohn, and Awan}
\firstpageno{1}

\begin{document}
\title{Differentially Private Topological Data Analysis}

\author{\name Taegyu Kang$^*$ \email kang426@purdue.edu \\
       \addr Department of Statistics\\
       Purdue University\\
       West Lafayette, IN 47907, USA
       \AND
       \name Sehwan Kim$^*$ \email kim3009@purdue.edu \\
       \addr Department of Statistics\\
       Purdue University\\
       West Lafayette, IN 47907, USA
              \AND
       \name Jinwon Sohn$^*$\email sohn24@purdue.edu \\
       \addr Department of Statistics\\
       Purdue University\\
       West Lafayette, IN 47907, USA
              \AND
       \name Jordan Awan$^{\dag }$ \email jawan@purdue.edu \\
       \addr Department of Statistics\\
       Purdue University\\
       West Lafayette, IN 47907, USA
       }
\footnote[0]{*: Kang, Kim, and Sohn are co-first authors and they
contribute equally to this paper. $\dag$: Corresponding author.}

\editor{}

\maketitle

\begin{abstract}%
This paper is the first to attempt differentially private (DP) topological data analysis (TDA), producing near-optimal private persistence diagrams. We analyze the sensitivity of persistence diagrams in terms of the bottleneck distance, and we show that the commonly used \v{C}ech complex has sensitivity that does not decrease as the sample size $n$ increases. This makes it challenging for the persistence diagrams of \v{C}ech complexes to be privatized. As an alternative, we show that the persistence diagram obtained by the $L^1$-distance to measure (DTM) has sensitivity $O(1/n)$. Based on the sensitivity analysis, we propose using the exponential mechanism whose utility function is defined in terms of the bottleneck distance of the $L^1$-DTM persistence diagrams. We also derive upper and lower bounds of the accuracy of our privacy mechanism; the obtained bounds indicate that the privacy error of our mechanism is near-optimal. We demonstrate the performance of our privatized persistence diagrams through simulations as well as on a real data set tracking human movement.
\end{abstract}


\section{Introduction}

Recent advances in technology make it possible to obtain data with such complicated structure that traditional data analysis methodologies cannot deal with them appropriately. To analyze with such complex data, topological data analysis has been an indispensable tool in data science \citep{Niyogi2011, khas;munc;16, Wasserman2018, dind;etal;20, riec;eta;20}. Essentially, topology is the most fundamental mathematical structure where the notion of ``\textit{nearness}'' can be discussed, and its generality makes it an appropriate framework for discussing extremely complicated data which are not expected to have more equipped structures such as vector spaces, manifolds, and so on. Topological data analysis is a novel branch of data analysis which was invented to capture the topological structure of data, and it has been deeply studied for the last couple of decades. See \cite{Carlsson2009} for a comprehensive overview. Especially, persistent homology, its flagship method, has been extensively studied theoretically and applied to many different disciplines such as medicine \citep{Nicolau2011}, biology \citep{McGuirl2020},  neuroscience \citep{Xu2021, Caputi2021}, astronomy \citep{Xu2019}, and machine learning \citep{Hensel2021,Betthauser2022}, to name a few.

At the same time, as bigger and more diverse data have become accessible, the issue of protecting private information of individuals in the data has also gained attention. Due to this concern, there is an increasing demand for privacy protecting procedures with formal guarantees. Such a paradigm has accelerated the growing attention to a well-formulated framework of privacy protection in data science. Differential privacy \citep{dwork2006calibrating} is the state-of-the-art framework that formally quantifies the notion of privacy and its protection.  Differential privacy requires that a privacy-protecting algorithm produces similar results for any two data sets, which differ at only one data point. The exact definition of $\epsilon$-DP will be introduced in the following section and we recommend \cite{Dwork2014} for a comprehensive introduction to DP. Recently, DP has been one of the central research topics in data science due, tackling problems in deep learning \citep{shokri2015privacy,abadi2016deep}, functional data analysis \citep{hall2013differential,mirshani2019formal}, social networks \citep{karwa2016inference,karwa2017sharing}, as well as many others.

While the DP framework has been widely adapted to numerous methodologies in data science as mentioned above, its application to TDA has yet to be discussed. To the best of our knowledge, the only work involved with both DP and TDA is \cite{Hehir;Vishwanath;Slakovic;Niu:2022}, which solely used persistence diagrams as a method of communicating the utility of a randomized response algorithm, and did not attempt to produce a private version of a TDA object. 
We believe that introducing the DP framework to TDA will be an emerging direction of research because many areas where TDA methods have been successfully utilized use data containing people's sensitive information. For example, \cite{Shnier:2019} applied persistence diagrams to differentiate gene expressions in individuals with autism spectrum disorders from those in a control group. 
Furthermore, TDA methods are used in several other problems in medical domain and neuroscience, as mentioned above, such as brain connectivity \citep{Caputi2021}, breast cancer \citep{Nicolau2011}, and neurological disorder \citep{Lee_et_al:2011}. 
Finally, TDA has recently been combined with other popular machine learning methods such as convolutional neural networks \citep{Love;Filippenko;Maroulas;Carlsson:2023}, auto-encoders \citep{hofe;etal;19}, etc. Hence, introducing DP to TDA may have far-reaching influence in data science.


\vspace{2mm}

{\bf Our Contributions}:  This paper is concerned with how to introduce the concept of differential privacy (DP) into the framework of topological data analysis (TDA). Our key observation is that, to exploit currently available privacy mechanisms, one needs an outlier-robust TDA method. 
Such an observation agrees with a long-standing intuitive principle in differential privacy saying that the specific data of any one individual should not have a significant effect on the outcome of the analysis to achieve privacy protection; for instance, see \cite{Dwork2009}, \cite{avella2021privacy}. To illuminate the adaptation of this principle to TDA, we examine the sensitivity of the bottleneck distance of persistence diagrams, which is the most widely used presentation of persistent homology, obtained by two different types of construction: persistence diagrams obtained from \v{C}ech complexes, which we see is not outlier-robust; and persistence diagrams obtained from the distance to a measure (DTM), which is outlier-robust. Our examination shows why persistence diagrams of \v{C}ech complexes are not readily privatized, and how persistence diagrams of the DTM can overcome such a difficulty. Moreover, we discuss how the magnitude of outlier-robustness affects the rate of sensitivity of the bottleneck distance, and propose to use $L^1$-DTM in order to achieve a minimal sensitivity. Based on the sensitivity analysis, we propose the first differentially private mechanism for persistence diagrams that provides $\epsilon$-differential privacy, using the exponential mechanism. We also establish upper and lower bounds for the accuracy error of our mechanism. The established bounds indicate that the privacy error of our mechanism is near-optimal. Our contributions can be summarized more specifically as follows:

\begin{itemize}
\item We prove that the sensitivity of the persistence diagram of \v{C}ech complexes, defined in terms of the bottleneck distance, does not diminish to zero as the sample size increases.

\item We propose using the persistence diagram of the distance to measure (DTM) as an alternative, and we prove that the $L^p$-DTM persistence diagram is guaranteed to have sensitivity, which is defined in terms of the bottleneck distance, $O(n^{-1/p})$. This leads us to use the $L^1$-DTM persistence diagram that guarantees the sensitivity $O(n^{-1})$.

\item We apply the exponential mechanism whose utility function is defined in terms of the bottleneck distance of $L^1$-DTM persistence diagrams in order to produce differentially privatized persistence diagrams. To the best of our knowledge, our algorithm is the first attempt of developing a mechanism generating differentially privatized persistence diagrams. We also find upper and lower bounds of the accuracy error of our mechanism.

\item We prove that any privacy mechanism applied to the $L^1$-DTM persistence diagrams cannot have accuracy, whose decay order is superior to the upper bound of the decay order of the privacy error corresponding to our mechanism. This result indicates that our mechanism may have optimal privacy error.
\end{itemize}

{\bf Organization:} 
The remainder of the paper is organized as follows.  In Section \ref{sec: prelim}, we briefly review the background and notation of TDA and DP. In Section \ref{sec: sensitivity}, we first examine the sensitivity of persistence diagrams constructed from the \v{C}ech complexes, as well as for an outlier-robust construction of persistence diagrams obtained from the DTM, introduced by \cite{Chazal;Cohen-Setiner;Merigot:2011}. In Section \ref{sec: ExpMech}, based on the sensitivity analysis given in Section \ref{sec: sensitivity}, we employ the exponential mechanism to generate privatized persistence diagrams. We also derive upper and lower bounds of its accuracy. Simulation studies which implement our algorithm are given in Section \ref{sec: simulation}. In Section \ref{sec: real data}, we apply our algorithm to a real-world data set including information about the locations of three people walking in a building recorded on smartphones over time. All proofs as well as additional results of real data analyses are presented in the appendices.

\section{Preliminaries} \label{sec: prelim}
 In this section, we introduce the persistence diagram, which is a statistic about the shape of the data, and look at bottleneck distance, a metric in the persistence diagram space as well as its stability. Also, we review the $\ep$-differential privacy ($\ep$-DP) and the exponential mechanism, which is one of the algorithms that satisfies $\ep$-DP. 
 
\paragraph{Notation} Throughout the paper, for real numbers $A$ and $B$ which possibly depend on a parameter $n \in \mathbb{Z}_+$, we use the asymptotic notation $A \lesssim B$ or $A = O(B)$ to denote the bound $|A| \leqslant CB$ for some absolute constant $C > 0$. If	the constant $C$ depends on some parameters, we will explicitly indicate them; for instance, $A \lesssim_{k,d} B$ means the bound $|A| \leqslant C_{k,d} B$ with a constant $C_{k,d}$ depending only on $k$ and $d$. If $A \lesssim B$ and $A \gtrsim B$, we write denote it by $A \approx B$. We also use the notation $A = o(B)$ to denote the asymptotic $\lim_{n \to \infty} A/B = 0$. In addition, for random variables $X$ and $Y$ which possibly depend on a parameter $n \in \mathbb{Z}_+$, we write $X = O_p(Y)$ to mean that $X/Y$ is bounded in probability and $X = o_p(Y)$ to mean that $X/Y$ converges to zero in probability, which are standard notations in probability theory. Let $(X_n)_{n=1}^{\infty}$ be a sequence of random variables. We write $X_n = \tilde{O}_p \big( f(n) \big)$ to mean that $X_n = O_p \big(f (n) \log^k n \big)$ for some $k \in \mathbb{Z}_+$. Basically, all such notations describe the asymptotic relationships in terms of large enough $n$. 

For a given metric space $(\mathcal{X}, \mathrm{d})$, $\mathcal{D}_n := \mathcal{D}_n(\mathcal{X}) := \mathcal{X}^n$ denotes the set of all $n$-tuples of elements in $\mathcal{X}$ for every $n \in \mathbb{Z}_+$. 

\subsection{Persistent Homology and Diagrams}
\label{sec: prelim-persistent homology}

Here, we briefly introduce two methods of constructing persistent homology and corresponding persistence diagrams of data, which will show up in our main discussion. The former one is the persistent homology of \v{C}ech complex and the latter one is the persistent homology of the sub-level sets of a continuous function. We believe that an intuitive and illustrative description of persistent homology will suffice to understand the results of this paper. More detailed background knowledge about persistent homology along with some fundamental knowledge about simplicial homology is presented in Appendix \ref{appendix: TDA}. For a deeper and comprehensive understanding for persistent homology, we refer the reader to the literature of persistent homology; for instance, \cite{Edelsbrunner2008, Edelsbrunner2009, Zomorodian2005}. For fundamental concepts about algebraic topology, we refer the reader to standard texts in algebraic topology such as \cite{Munkres1984, Bredon1997}.

Let $D = \{x_1, \dots, x_n\}$ be a finite subset of a metric space $(\mathcal{X}, d)$. Let $r >0$ be a positive real number. At every point $x_i$, we place a ball $B(x_i; r)$ with radius $r$ centered at $x_i$. The persistence homology of the \v{C}ech complexes on $D$ captures the evolution of the homological structure of the union $\cup_{j=1}^n B(x_j; r)$ as $r$ varies. For instance, the $0$th homological feature represents the connected components of it and the $1$st homological feature represents the loops in it. Figure~\ref{fig:homoex} portrays how to construct the persistent homology of \v{C}ech complexes. As the radius $r$ varies, some homological features show up and disappear, and such ``birth'' and ``death'' of homological features are presented as multisets called persistence diagrams. More precisely, a $q$th persistence diagram of the \v{C}ech complexes on the data set $D$ a multiset that consists of finitely many, say $m$, points $(b_i, d_i)$ satisfying $0 \leq b_i \leq d_i \leq \infty$ for every $i = 1, \dots, m$; the presence of each point $(b_i, d_i)$ means that there exists a $q$-dimensional homological feature that shows up at radius $b_i$ and disappears at radius $d_i$. 

As for the other method, let $f_D : \mathcal{X} \to \mathbb{R}$ be a continuous function defined on metric space $(\mathcal{X}, d)$, possibly depending on the given set $D$. For each $r \in \mathbb{R}$, one can consider the sub-level set $L_r := \{x \in \mathcal{X} \: : \: f_D(x) \leqslant r\}$. As we consider the evolution of the union of balls in the previous way of construction, we now consider the evolution of the sub-level sets $L_r$ as $r$ varies. Figure \ref{fig1: contruction} illustrates such an evolution of a certain continuous function. In the figure, three connected components and one loop show up once and disappear at some time except for a single connected component. The birth-death pairs at each dimension can be presented as a persistence diagram, just as for the \v Cech complex. 

In general, let a filtration of topological spaces $\{U_r\}_{r \in R}$
be given, where $R$ is a linearly ordered set; that is, for any $r_1$ and $r_2$ in $R$ satisfying $r_1 \leq r_2$, $U_{r_1} \subseteq U_{r_2}$. Then, one can define the persistent homology and the corresponding persistence diagram of the sequence. In our first example, each $U_r$ is the union of balls with radius $r$ (or, the simplicial complex obtained from the balls); in our second example, each $U_r$ is the sub-level set $L_{r}$.

\begin{figure}
    \centering
    \includegraphics[width=\textwidth]{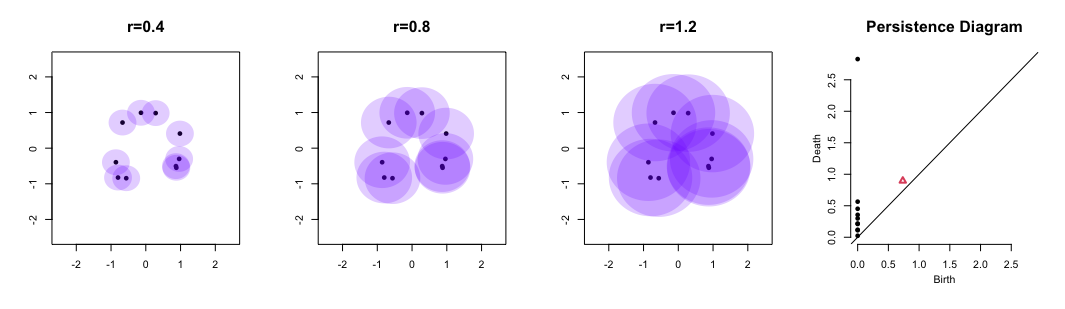}
    \caption{Constructing of \v{C}ech complexes and its persistence diagram: The left three figures illustrate how \v{C}ech complexes on nine points supported on a circle are constructed. When $r = 0.4$, there are several connected components; but there is no loop. When $r= 0.8$, there exists a $1$-dimensional loop that captures the shape of the circle. When $r = 1.2$, the loop disappears and there is only a single contractible connected component. The right-most figure is the persistence diagram of the \v{C}ech complexes. Each black dot represents the birth-and-death times of each connected component and the red triangle represents the birth-and-death times of the loop.}
    \label{fig:homoex}
\end{figure}

\begin{figure}
    \centering
    \includegraphics[width=\textwidth]{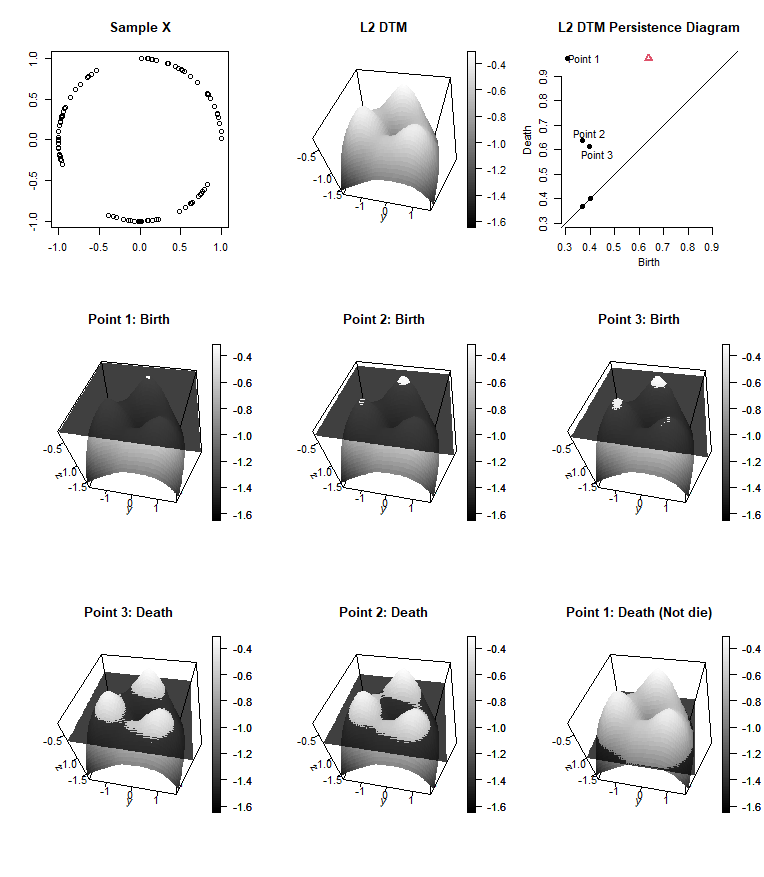}
    \caption{Filtration corresponding to the $L^2$-DTM of a circle data set: The data set is supported on a circle, and the $L^2$-DTM function of the data is visualized (for convenience, the function multiplied by $-1$ is presented). The persistence diagram constructed from the function is presented. The second and third columns present how the filtration of the sub-level sets of the function evolves. As the filtration evolves, three connected components show up at values $0.308$, $0.367$, and $0.397$ respectively. One component dies at a value of $0.613$, another one dies at $0.636$. The last one lives until the end of evolution.}
    \label{fig1: contruction}
\end{figure}

\subsection{Stability of Persistence Diagrams in the Bottleneck Distance}
\label{sec: prelim-stability}

A persistence diagram $\mathcal{P} = \{(b_i, d_i)\}_{i=1}^m$ is essentially a multiset of birth-death pairs $b_i$ and $d_i$, which satisfy $b_i \leq d_i$. There are numerous ways to ``vectorize'' a persistence diagram into an element in some vector space. One of the most popular ways is to represent each birth-death pair $(b, d)$ by the Dirac measure $\delta_{(b, d)}$ at $(b, d)$, and represent the whole diagram $\mathcal{P}$ by the point measure $\sum_{i=1}^m \delta_{(b_i, d_i)}$ which is a measure on the set $\mathcal{T} := \{(x, y) \: : \: 0 \leq x \leq y \leq \infty \}$ (for a detailed description of such a way of vectorization, see Section 2 in \cite{Owada2021}). By realizing a persistence diagram as a measure, it is possible to define the distance between two persistence diagrams by means of a distance between measures. One of the most popular choices is using the $L^{\infty}$ Wasserstein distance of the measures, which is called the bottleneck distance. Specifically, let $\mathcal{P}, \mathcal{P}'$ be two persistence diagrams. Then the bottleneck distance between $\mathcal{P}$ and $\mathcal{P}'$ is defined as
\[
\db (\mathcal{P}, \mathcal{P}') := \min_{g : \breve{\mathcal{P}} \leftrightarrow \breve{\mathcal{P}}'} \max_{z \in \breve{\mathcal{P}}} \parallel z - g(z) \parallel_{\infty},
\]
where $\breve{\mathcal{P}}$ and $\breve{\mathcal{P}}'$ denote the persistence diagrams $\mathcal{P}$ and $\mathcal{P}'$ along with the copies of all points on the diagonal respectively; $g: \breve{\mathcal{P}} \leftrightarrow \breve{\mathcal{P}}'$ ranges over all bijections between $\breve{\mathcal{P}}$ and $\breve{\mathcal{P}}'$. In words, $\db (\mathcal{P}, \mathcal{P}')$ the minimax cost of pairing the birth-death points in one diagram to the other diagram one-by-one in terms of $\ell_{\infty}$ distance. When two diagrams contain different numbers of birth-death points, then the remaining points in one diagram pair up with the points on the diagonal.

A key property of the bottleneck distance is the following stability property (for more details, see \cite{Cohen-Steiner2007, Chazal2016_book}). Suppose that $\mathcal{P}_q(D)$ and $\mathcal{P}_q(D')$ are $q$th persistence diagrams constructed from the \v{C}ech complexes of two sets $D$ and $D'$ in a metric space $(\mathcal{X}, d)$, then
\[ \label{exp: bottleneck stability Cech} \tag{2.1}
\db (\mathcal{P}_q (D), \mathcal{P}_q ( D')) \lesssim d_H(D, D'),
\]
where

\[
d_H(D, D') := \max \big\{ \sup_{x \in D} \inf_{y \in D'} d(x, y), \sup_{y \in D'} \inf_{x \in D} d(x, y) \big\}
\]
denotes the Hausdorff distance between $D$ and $D'$. Analogously, if $\mathcal{P}_q(D)$ and $\mathcal{P}_q(D')$ are obtained from the filtrations of the sub-level sets of continuous tame functions $f_{D}$ and $f_{D'}$, respectively. Then
\[ \label{exp: bottlneck stability} \tag{2.2}
\db (\mathcal{P}_q(D), \mathcal{P}_q(D'))) \leq \sup_{x \in \mathcal{X}} |f_{D}(x) - f_{D'}(x)|.
\]
The precise definition of tame functions is presented in Definition \ref{def: tame}. For more comprehensive discussion, please refer to \cite{Cohen-Steiner2007}. Intuitively, a $\mathbb{R}$-valued function $f$ is said to be tame if the homology of its sub-level sets changes at most finitely many times.

\subsection{Differential Privacy}

 DP is a mathematical framework designed to quantify the privacy leakage of a proposed randomized algorithm (called a mechanism), introduced by \citet{dwork2006calibrating}.  The first step for measuring such privacy risk starts from specifying which databases are considered to ``differ in one entry,'' which we refer to as \emph{adjacent databases}. 
We say $D$ and $D'$ are adjacent if $d(D,D')\leq 1$, for some metric $d(\cdot,\cdot)$ between databases.  In this paper, we use  Hamming distance $H(\cdot,\cdot)$, which counts the number of entries that differ between $D$ and $D'$. 
A privacy mechanism $\mathcal{M}:\mathcal{D}_n \rightarrow \mathcal{Y}$  returns a random variable $\mathcal{M}(D)$ for any $D \in \mathcal{D}_n$, and the privacy risk of the algorithm $\mathcal{M}$ can be evaluated by definition as follows:

\begin{definition}[$\epsilon$-Differential privacy ($\epsilon$-DP): \citealp{dwork2006calibrating}] Given $\epsilon \geq 0$, a privacy mechanism $\mathcal{M}$ on the output space $\mathcal{Y}$ satisfies $\epsilon$-DP if  
\[
\label{exp: pure DP} \tag{2.2}
\mathbb{P} (\mathcal{M}(D)\in S)\leq e^{\epsilon} \mathbb{P}(\mathcal{M}(D')\in S), 
\]
for every measurable set $S \subset \mathcal{Y}$ and all $D$ and $D'$ satisfying $H(D, D') \leq 1$.
\end{definition}

This definition characterizes how much privacy leakage could occur via the privacy budget parameter $\epsilon$ when a single entity in $D$ is not the same one in $D'$. 
The smaller that $\epsilon$ is, the harder it is to distinguish the probability distributions of $\mathcal{M}(D)$ and $\mathcal{M}(D')$, which accordingly makes it harder to identify whether data set $D$ or $D'$ was used in the analysis by $\mathcal{M}$ \citep{wasserman;zhou;2010}.

For conceptual understanding, let us imagine that a data set $D$ contains information of 100 people obtained by a survey. 
Let us call one person in $D$ Person A, and let us assume that a $\epsilon$-DP privacy mechanism $\mathcal{M}$ is employed so that a summary from $\mathcal{M}(D)$ is released, containing some useful information about $D$. Then, it is known that, for any other data set $D'$ that shares 99 people with $D$, except for Person A, the probability distributions of $\mathcal{M}(D)$ and $\mathcal{M}(D')$ cannot be easily distinguished. Thus, it is difficult to identify whether Person A is included in the actual data set $D$ or not. Notice that Person A was chosen arbitrarily in $D$, so the privacy of each individual in $D$ is protected.

While any $\ep$-DP mechanism preserves privacy, not all mechanisms ensure good performance with respect to an underlying utility. It is straightforward to imagine that sanitized statistics can devastate the performance of the utility due to excessive noises for privacy. In contrast, the exponential mechanism is a general technique that takes care of the utility while being able to control the privacy leakage within the budget $\epsilon$.

\begin{proposition}[Exponential mechanism: \citealp{McSherry2007exponential}] 
\label{prop: ExpMech ep-DP}
Let $n \in \mathbb{Z}_+$ and let $\{u_{D} : {\cal Y} \rightarrow {\mathbb R} : D \in \mathcal{D}_n \}$ be a collection of utility functions. Assume that the \emph{sensitivity} $\Delta(u)$ is finite:
\[ 
\label{exp: ExpMech-sensitivity} \tag{2.3}
\Delta(u) = \underset{H(D,D')\leq 1}{\sup}~~~\underset{y\in {\cal Y}}{\sup}~~~|u_D(y)-u_{D'}(y)| < \infty, 
\]
where the supremum is over all adjacent $D$ and $D'$ and assume that $\int \exp\left(u_D(y)\right) d\nu(y) < \infty$ for all $D\in {\cal D}$ where $\nu$ is a measure in ${\cal Y}$. If $\Delta$ satisfies $\Delta(u)\leq \Delta<\infty$, then the collection of mechanisms $\{{\cal M}(D): D\in{\cal D}\}$, each of which has the probability density with respect to $\nu$
\begin{align*} \label{exp: density} \tag{2.4}
p_D(y) \propto \exp\left(\dfrac{\epsilon}{2\Delta}u_D(y)\right), 
\end{align*}
satisfies $\epsilon$-DP.
\end{proposition}

The exponential mechanism can be easily applied to a wide variety of problems having utility functions. One of the simplest examples is a count statistic. Let us define the count statistic $\mathrm{count}(D)$ of a data set $D$ to be the number of data points in $D$ having a certain property, and define a utility function $u_D(y) := -|y - \text{count}(D)|$, which has sensitivity 1. This utility puts higher values when $y$ is close to $\text{count}(D)$. 
Many machine learning and statistical inference problems can be privately handled using the exponential mechanism when it is possible to define appropriate utility functions such as empirical risk or likelihood functions \citep{huan;etal;12,awan2019benefits,cumm;etal;19,lu;etal;22}.

\begin{proposition}[Utility of the exponential mechanism: \citealp{Dwork2014}]\label{prop: ExpMechUtility}
Let $\mathrm{OPT}_{D}=\max_{y\in \mathcal{Y}}~u_D(y)$ be the optimal value that can be achieved by the utility function over all outputs, given database $D$. Let $Y$ be a random variable with the density given in (\ref{exp: density}). Then,
\[ 
\label{exp: ExpMech-utility} \tag{2.5}
\mathbb{P} \bigg[ u_D(Y) \leq \mathrm{OPT}_D - \frac{2 \Delta}{\epsilon} \big( \log |\mathcal{Y}| + t \big) \bigg] \leq e^{-t},
\]
for every $t \geq 0$. Consequently, 
\[ \label{exp: ExpMech-utility O_p} \tag{2.6}
u_D(Y) = \mathrm{OPT}_D + O_p \bigg( \frac{\Delta \log |\mathcal{Y}|}{\epsilon} \bigg).
\]
\end{proposition}


The utility function in the exponential mechanism must be carefully chosen to ensure that the error rate given in Proposition \ref{prop: ExpMechUtility} translates to optimal rates for the private output. For example, \citet{awan2019benefits} showed that when the utility function has a quadratic Taylor expansion at its maximum,  the randomness for privacy in the exponential mechanism often gives rise to $O_p(1/\sqrt{n})$ noise, which in general is of the same order as the non-private statistical estimation problems. On the other hand, \citet{reimherr2019kng} showed that for some utility functions which are locally approximated by the absolute value function, the randomness for privacy may be as low as $O_p(1/n)$. 

With the goal of producing differentially private persistence diagrams, we propose using exponential mechanism whose utility function is the negative bottleneck distance between the private and non-private persistence diagrams.

\section{Sensitivity of Persistence Diagrams in the Bottleneck Distance} \label{sec: sensitivity}

Most DP algorithms require quantifying how much the value of a statistic is changed by changing a single point in a given data. The largest possible amount of that change in the statistic is colloquially called the sensitivity of the statistic. In this study, we regard a persistence diagram constructed from a data set $D$ as a statistic that estimates the homological structure of the space underlying the data, and we use the bottleneck distance to define a metric on the space of persistence diagrams. Hence, to apply a DP mechanism to persistence diagrams, our first step should be estimating the sensitivity of persistence diagrams in terms of the bottleneck distance; namely, we are going to analyze how big the bottleneck distance $\db(\mathcal{P}_D, \mathcal{P}_{D'})$ can be, where the pair $(D, D')$ denotes a pair of adjacent data sets. Note $\mathcal{P}_D$ and $\mathcal{P}_{D'}$ mean the persistence diagrams constructed from the data sets $D$ and $D'$ respectively under a given way of constructing persistent homology. 

 In differential privacy, to ensure consistent estimators, it is necessary that the sensitivity goes to $0$ as the size of the data grows. Our key observation is that if a chosen way of constructing persistent homology is not outlier-robust, the sensitivity of the corresponding persistence diagrams may not tend to $0$ even if the size of data, say $n$, grows.

We demonstrate that the sensitivity of the persistence diagrams of \v{C}ech complexes cannot converge to $0$ even if the size of data grows to infinity. To overcome such an issue, we propose using the notion called distance to a measure (DTM), which was thoroughly discussed by \cite{Chazal2018} to give birth to outlier-robust persistence diagrams. Moreover, among various versions of construction of DTM, we propose using $L^1$-DTM which gives the smallest sensitivity.

Before moving on to the main sensitivity analysis, we would like to make the terminologies clear. In the introduction to this section, we have been using the word sensitivity for two different quantities: sensitivity of the bottleneck distance and the sensitivity of utility functions of the exponential mechanism. To avoid a confusion, going forward we refer to the sensitivity of the bottleneck distance of persistence diagrams as the \textit{base sensitivity}, which is the terminology introduced in \cite{Awan;Wang:2022}. The base sensitivity of the bottleneck distance of $q$th persistence diagrams from \v{C}ech complexes is denoted by $\Delta^{\mathrm{\breve{C}ech}}_q$ and that from the $L^1$-DTM is denoted by $\Delta^{\mathrm{DTM}}_q$. The precise definition of them will be presented in each of the following subsections. Otherwise, we will reserve the term ``sensitivity'' for the sensitivity of a given utility function of the exponential mechanism. 

\subsection{Sensitivity of the Persistence Diagrams of \v{C}ech Complexes} \label{sec: sensitivity-Cech}

Let us illustrate how the construction of \v{C}ech complexes fails to have a decreasing base sensitivity. The situation of the following example is well illustrated in Figure \ref{fig: two clusters}. Note that Figure \ref{fig: two clusters} draws figures by means of the Vietoris-Rips complex instead of the \v{C}ech complex. The Vietoris-Rips complex is a variant of the \v{C}ech complex which has a computational advantage than the \v{C}ech complex. In fact, the filtration of Vietoris-Rips complexes has essentially the same information with that of \v{C}ech complexes. The definition of the Vietoris-Rips complex and its relationship with the \v{C}ech complex is presented in Appendix \ref{appendix: TDA-3}.

\begin{example}\label{ex: cech}
Let $D$ be a set of $n$ points in $\mathbb{R}^2$ that is tightly clustered into exactly two clusters. Write $x$ to denote the point located at the midpoint of the clusters, and take $D'$ to be the data set obtained by moving one point in $D$ to $x$. Now, further imagine that $n$ grows while the configuration of the points in $D$ and $D'$ remains the same, and derive the $0$th dimensional persistence diagrams obtained from the \v{C}ech complexes of $D$ and $D'$. Then, the connected components in $D$ collapse into the two clusters quickly, while the isolated point $x \in D'$ produces an additional connected component that lives longer. Such a discrepancy between two persistence diagrams prohibits the bottleneck distance between them from going to $0$. More precisely, the bottleneck distance between them remains as big as the distance of the point $x$ from the clusters in $D$. 
\end{example}

The following theorem establishes that this phenomenon is widespread. We denote the $q$th persistence diagram constructed from the \v{C}ech complexes on the data set $D$ by $\mathcal{P}_q^{\text{\v{C}ech}}(D)$.

\begin{lemma} \label{lem: Cech sensitivity lower}
Let $D = \{x_1, \dots, x_n\}$ be a subset of an Euclidean space $\mathbb{R}^d$. Let $\{d_1, \dots, d_m\}$ be the set of distinct finite death times in $\PCech(D)$ with $0 < d_1 < \dots < d_m < \infty$. Let $\delta = d_m - d_{m-1}$ (if $m = 1$, let $\delta = d_1$). Then, it is possible to take a set $D'$ with $|D \setminus D'| + |D' \setminus D| \leqslant 1$ satisfying that
\[
\db \left(\PCech(D), \PCech(D')\right) \geq \min\{\delta, d_m / 2 \}.
\]
\end{lemma}

Roughly, the theorem can be proved by constructing a data set $D'$ having an additional point at the middle of the most ``significant connected components'' in the filtration of \v{C}ech complexes of $D$, i.e., the connected components that die at time $d_m$. The detailed proof is presented in Appendix \ref{appendix: proofs-section3}.

From now on, we assume that all the data sets are supported in a bounded subset $E$ of $\mathbb{R}^d$ unless there is any additional specification.  We define the base sensitivity $\Delta^{\mathrm{\breve{C}ech}}_q$ concerning \v{C}ech complexes:
\[
\Delta^{\mathrm{\breve{C}ech}}_q := \sup_{H(D, D') \leq 1} \db \left( \PCechq (D), \PCechq (D') \right).
\]

Note that the stability theorem (\ref{exp: bottleneck stability Cech}) implies the following upper bound of the base sensitivity:
\[
\Delta_q^{\mathrm{\breve{C}ech}} \leq \mathrm{diam} E
\]
for every non-negative integer $q$. Lemma \ref{lem: Cech sensitivity lower} provides the matching lower bound of the base sensitivity for $q = 0$. Moreover such  upper and lower bounds show that the sensitivity of the utility function $v_D$ defined as
\[ \label{exp: util Cech} \tag{3.1}
v_D(\mathcal{P}) := - \db \left( \PCech (D), \mathcal{P} \right),
\]
has sensitivity of constant order:

\begin{theorem}\label{thm: Cech sensitivity lower}
Suppose that a given data generating process is supported on a bounded subset $E$ of a Euclidean space. Then, we have
\[
\Delta_0^{\mathrm{\breve{C}ech}} \geq \frac{\mathrm{diam} E}{4}.
\]
Moreover, the utility function $v_D$ defined in (\ref{exp: util Cech}) satisfies
\[
\frac{1}{4} \mathrm{diam} E \leq \sup_{H(D, D') \leq 1} \sup_{\mathcal{P}} |v_D(\mathcal{P}) - v_{D'}(\mathcal{P})| \leq \mathrm{diam} E
\]
\end{theorem}

Theorem \ref{thm: Cech sensitivity lower} shows that why it is challenging to develop a privacy mechanism for \v{C}ech complexes: \v{C}ech complexes are so sensitive, in terms of the bottleneck distance of their persistence diagrams, that the sensitivity of the utility function $v_D(\cdot)$ remains constant regardless of the size $n$ of the data set. This implies that the exponential mechanism using this utility function keeps adding a constant size of noise even if $n$ gets bigger. This prevents the bottleneck distance from becoming small even in the case of huge $n$.

\begin{figure}[t]
\centering
        \psfrag{x1}{$x$}
        \psfrag{x2}{$y$}
        \psfrag{i}{(I)}
        \psfrag{ii}{(II)}
        \psfrag{ii}{(III)}
      \includegraphics[width=1.0\textwidth]{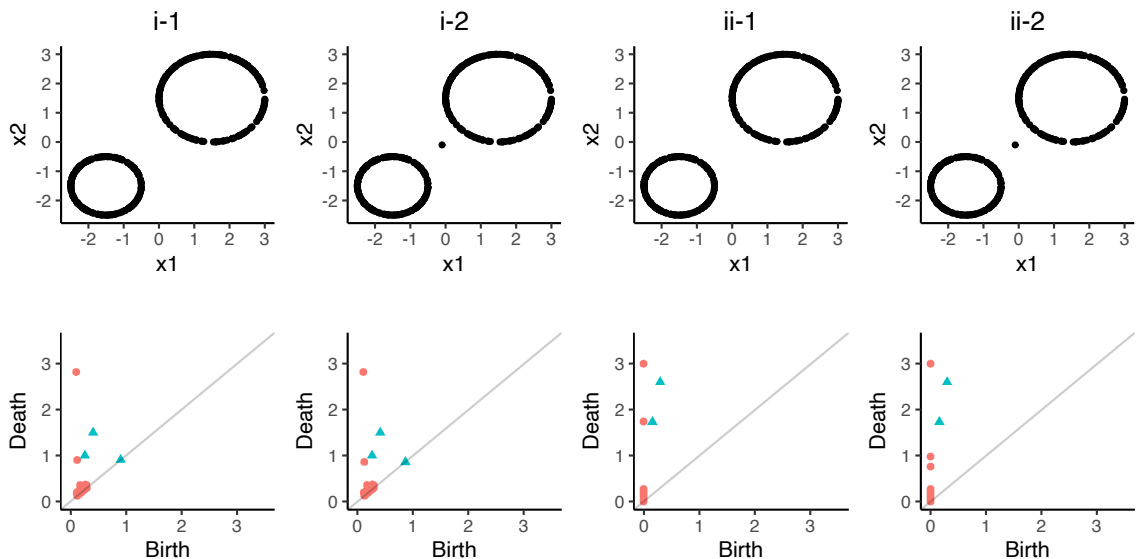}
\caption{Persistence diagrams on $D$ and $D'$: the red circles and the green triangles are the connected components and the loops respectively. The columns (i-1) and (i-2) correspond to the results with $L^1$-DTM on $D$ and $D'$ respectively, and the two diagrams have 0.042 bottleneck distance in terms of the connected components. The columns (ii-1) and (ii-2) are from the Vietoris-Rips complex and the distance between the two diagrams have 0.762.
} 
\label{fig: two clusters}
\end{figure}

\subsection{Sensitivity of the Persistence Diagrams of the DTM} \label{sec: sensitivity-DTM}

The DTM, which was introduced by \cite{Chazal;Cohen-Setiner;Merigot:2011}, provided a novel way to overcome the sensitivity to outliers. 
DTM proposes measuring how far each point is from the dense part of the support of the probability measure. By doing so, an outlier corresponds to a relatively large distance. Thus, when it comes to concerning the filtration of the sub-level sets of a DTM, the topological features produced by the outlier would occur in the late period of the filtration, or it might not occur through the whole filtration. More thorough discussions on the DTM can be found in \cite{Chazal2018, Anai2020, Oudot:2015}.  By virtue of the properties of the DTM, it is very likely that DTM-based persistence diagrams give rise to a much smaller sensitivity, so it may provide us with a suitable TDA statistic to build our privatized mechansim upon. In fact, we show that DTM-based persistence diagrams achieve sensitivity converging to $0$ as $n$ grows to infinity, but the rate of decay depends on which class of DTM designs we use. 

Basically, a DTM is defined to be a $L^p$ norm of a certain function. The original version of DTM was defined to be a $L^2$-type quantity. We show that the $L^2$-type DTM produces persistence diagrams whose base sensitivity is bounded by $O(n^{-1/2})$, and we recognize that each $L^p$-DTM results in an analogous upper bound of the base sensitivity: $O(n^{-1/p})$. From this observation, we focus on the $L^1$-DTM that has the fastest decay rate in the base sensitivity. Furthermore, we also verify the base sensitivity of the persistence diagrams obtained from the $L^1$-DTM is bounded below by $n^{-1}$ up to a constant. In other words, our sensitivity analysis for $L^1$-DTM is sharp up to constants.

We present the definition of the general $L^p$-DTM and its empirical realization. The key property to obtain  upper bounds of the persistence diagrams is the so-called Wasserstein stability of a DTM, which was extensively discussed in the past literature; for instance, see \cite{Chazal2016}. As a result of the Wasserstein stability, we deduce the upper bound of rate $n^{-1/p}$ for the $L^p$-DTM. The matching lower bound of rate $n^{-1}$ for the $L^1$-DTM is established by constructing a specific example that exactly gives the lower bound. All the proofs are presented in the appendix.

\begin{definition}[Distance to a measure]
Let $\mu$ be a probability measure and $X$ be a random variable whose probability distribution is $\mu$. For the given $\mu$, $0 < m < 1$, and $p \geq 1$, the $L^p$ distance to the measure $\mu$ at resolution $m$ is defined by
\[
\delta^{(p)} (x) := \delta_{\mu,m}^{(p)}(x) := \bigg[ \frac{1}{m} \int_{u = 0}^m \big( G_x^{-1}(u) \big)^p du \bigg]^{1/p},
\]
where $G_x(t) = P \big[ \parallel X - x \parallel \leq t \big]$. Here, $\parallel \cdot \parallel$ denotes the $\ell^2$-norm in Euclidean spaces.

\end{definition}

The hyperparameter $m$ determines how much smoothing effect will be employed, which is reminiscent of the role of the bandwidth in a kernel density estimation. A natural empirical approximation would be the following.

\begin{definition}[Empirical version of the DTM]
Let $X_1, \dots, X_n$ be i.i.d. samples obtained from a probability distribution $\mu$ and $\mu_n$ the empirical probability measure defined on this sample, i.e., 
\[
\mu_n := \frac{1}{n}  \sum_{i=1}^n \delta_{X_i}
\]
The empirical $L^p$-DTM to $\mu$ at resolution $m$, denoted by $\hat{\delta}^{(p)}$, is defined to be the $L^p$-DTM to $\mu_n$ at resolution $m$; namely,
\[
\hat{\delta}^{(p)}(x) := \delta_{\mu_n, m}^{(p)}(x) = \bigg[ \frac{1}{k} \sum_{X_i \in N_k(x)} \parallel X_i - x \parallel^p \bigg]^{1/p},
\]
where $k = \lceil mn \rceil$ and $N_k(x)$ is the set containing the $k$ nearest neighbours of $x$ among $X_1, \dots, X_n$. Here, the distance between data points is measured by the $\ell^2$-norm in Euclidean space.
\end{definition}

The key quantitative property of the $L^p$-DTM, which is called its Wasserstein stability, is the following: let $\mu$ and $\nu$ be probability measures defined on a common metric space, then
\[ \label{exp: Wasserstein stability} \tag{3.1}
\sup_x \left|\delta^{(p)}_{\mu}(x) - \delta^{(p)}_{\nu}(x)\right| \leq \frac{1}{m^{1/p}} W_p(\mu, \nu),
\]
where $W_p(\mu,\nu)$ denotes the $p$-Wasserstein distance between $\mu$ and $\nu$. For more details, see \cite{Chazal2016}. Let $D$ and $D'$ be adjacent data sets. 

Let $\PDTMq (D)$ denote the $q$th persistence diagram constructed from the filtration of sub-level sets of the $L^p$-DTM to the empirical distribution of the data set $D$. The base sensitivity of  $\Delta_q^{\mathrm{DTM}_p}$ concerning the DTM is
\[
\Delta_q^{\mathrm{DTM}_p} := \sup_{H(D, D') = 1} \db \left( \PDTMq (D), \PDTMq (D') \right).
\]
By virtue of the stability theorem (\ref{exp: bottlneck stability}) and the Wasserstein stability (\ref{exp: Wasserstein stability}), The following upper bound of the base sensitivity $\Delta_q^{\mathrm{DTM}_p}$ can be estalished by quantifying the $p$-Wasserstein distance between empirical distributions on adjacent data sets.

\begin{theorem}[Sensitivity of the persistence diagrams constructed from the $L^p$-DTM] \label{thm: DTM sensitivity upper}
Let $D$ and $D'$ be finite subsets of a bounded set $E$ in $\mathbb{R}^d$ satisfying $|D| = |D'| = n$ and $H(D, D') = 1$. Then, for every $q \in \mathbb{Z}^{\geq 0}$,
\[
\Delta_q^{\mathrm{DTM}_p} \leq \frac{\mathrm{diam} E} {m^{1/p} n^{1/p}}.
\]
\end{theorem}

In fact, as a result of the Wasserstein stability (\ref{exp: Wasserstein stability}), the result of the theorem can be obtained by estimating the $p$-Wasserstein distance between the empirical distributions on $D$ and $D'$. The detailed proof is presented in Appendix \ref{appendix: proofs-section3}.

According to Theorem \ref{thm: DTM sensitivity upper}, each $L^p$-DTM is guaranteed to have base sensitivity bounded above by $O(n^{-1/p})$. In particular, such a guaranteed upper bound becomes smallest when $p$ is taken to be $1$:
\[ \label{exp: L1-DTM sensitivity upper}
\db \left(\PDTMoneq (D), \PDTMoneq (D')\right) \leq \frac{\mathrm{diam} E}{m n},    
\]
In fact, the upper bound of the $L^1$-DTM is sharp up to constants.

\begin{proposition}[Lower bound of the sensitivity of the $L^1$-DTM] \label{prop: DTM sensitivity lower}
Suppose that $m < 1/2$. Then, for every positive integer $n$, there exists a pair of sets $D$ and $D'$ that satisfies $|D| = |D'| = n$, $H(D, D') = 1$, and 
\[
\db \left(\PDTMone (D), \PDTMone (D')\right) = \frac{\mathrm{diam} E}{2k},
\]
where $k = \lceil mn \rceil$.
\end{proposition}

The proof can be obtained by constructing a pair of adjacent data sets $D$ and $D'$ that achieve the proposed distance. In fact, the data sets illustrated in Figure \ref{fig: two clusters} achieve it. For a detailed proof, see Appendix \ref{appendix: proofs-section3}.

Now, we introduce the utility function that we use to design our privacy mechanism. Let $\mathsf{Pers}$ denote the space of persistence diagrams, equipped with the bottleneck distance. For any given data set $D$ and any non-negative integer $q$, we define the function $u_D^{(q)}: \mathsf{Pers} \to \mathbb{R}$ as follows:
\[
u_D^{(q)}(\mathcal{P}) := -\db \left(\mathcal{P}, \PDTMoneq (D)\right),
\]
Let $\ell$ be a chosen non-negative integer. Our utility function $u_D: \big( \mathsf{Pers} \big)^{\ell + 1} \to \mathbb{R}$ is defined as follows:
\[ \label{exp: utility DTM} \tag{3.2}
u_D(\mathcal{P}_0, \dots, \mathcal{P}_{\ell}) := \sum_{q = 0}^{\ell} u_D^{(q)}(\mathcal{P}_q).
\]
As a result of the upper and lower bounds for the base sensitivity, we can establish the following upper and lower bounds of the sensitivity:

\begin{corollary} \label{cor: DTM sensitivity}
For a chosen $\ell \geq 0$, let the utility function $u_D$ be defined as in (\ref{exp: utility DTM}). Then the following is satisfied.
\[
\frac{\mathrm{diam} E}{2 \lceil mn \rceil} \leq \sup_{H(D, D') = 1} \sup_{\mathcal{P} \in \mathsf{Pers}} |u_D(\mathcal{P}) - u_{D'}(\mathcal{P})| \leq (\ell + 1) \frac{\mathrm{diam} E}{mn}
\]
\end{corollary}

\begin{remark}
The additive nature of the utility function $u_D$ is what allows us to establish the upper and lower bounds in Corollary \ref{cor: DTM sensitivity}. Notice that the lower bound of the corollary is derived from the result of Proposition \ref{prop: DTM sensitivity lower} which is only valid for the $0$th persistence diagrams; but, the additive form of $u_D$ allows it to be a lower bound for the sensitivity of the entire utility function.
\end{remark}

\begin{remark}
Note that while the lower bound of $L^1$-DTM matches the rate of its upper bound, we do not at this time obtain such matching lower bounds of the other $L^p$-DTMs. Hence, it might be the case that the base sensitivity of the general $L^p$-DTM can be improved. For example, in the situation of Figure \ref{fig: two clusters}, we empirically verified that the bottleneck distance between $L^2$-DTM persistence diagrams of $D$ and $D'$ is also  $O(n^{-1})$. 
\end{remark}

\section{Employment of the Exponential Mechanism with the $L^1$-DTM} \label{sec: ExpMech}

In this section, we describe how to implement the exponential mechanism in order to generate privatized persistence diagrams. More specifically, the exact probability distribution from which we generate our privatized persistence diagrams is presented, and a Markov chain Monte Carlo procedure to approximate the distribution is summarized step-by-step. 

Let $D = \{X_1, \dots, X_n\}$ be a data set that consists of i.i.d. samples having a common probability distribution $\mu$. For brevity, we denote by $\mathcal{P}_q(D)$ for each $q \geq 0$ the $q$th persistence diagram obtained from the $L^1$-DTM to the empirical measure $\mu_n$, which was denoted by $\mathcal{P}_q^{\mathrm{DTM}_1}(D)$ in the previous section. And, we set $\mathcal{P}(D) := (\mathcal{P}_0(D), \dots, \mathcal{P}_{\ell}(D))$ for the given $\ell$. Let $\mathcal{P}_q(\mu)$ be the $q$th persistence diagram obtained from the $L^1$-DTM $\delta^{(1)}_{\mu}$ 
to the measure $\mu$, 
and let us define $\mathcal{P}(\mu) := (\mathcal{P}_0(\mu), \dots, \mathcal{P}_{\ell}(\mu))$ for the given $\ell$. Also, let $\mathcal{P}_{\mathrm{DP}} = (\mathcal{P}_{0, \mathrm{DP}}, \dots, \mathcal{P}_{\ell, \mathrm{DP}})$ be a tuple of privatized persistence diagrams generated from our algorithm. 

We analyze the error of our privatized persistence diagrams from two different points of view. First, we estimate $\db (\mathcal{P}(\mu), \mathcal{P}_{\mathrm{DP}})$. This quantity represents the error of the privatized persistence diagrams from the persistence diagram of the original data-generating process. From a statistical perspective, $\mathcal{P} (\mu)$ can be regarded as a parameter characterizing the true data-generating process. Hence, the quantity $\db (\mathcal{P} (\mu), \mathcal{P}_{\mathrm{DP}})$ quantifies the amount of error in estimating the parameter $\mathcal{P}(\mu)$ by the privatized statistic $\mathcal{P}_{\mathrm{DP}}$ that is obtained by privatizing the actual statistic $\mathcal{P}(D)$. The second approach is to estimate the quantity $\db (\mathcal{P}(D), \mathcal{P}_{\text{DP}})$ which quantifies how much the privatization process distorts the original non-privatized statistic.

\subsection{Generating Privatized Persistence Diagrams} \label{sec: ExpMech-algorithm}

The design of an exponential mechanism is formulated by specifying an output space $\mathcal{Y}$, and a utility function $u_D: \mathcal{Y} \to \mathbb{R}$ for each data set $D$. Since our target to be privatized is a persistence diagram, it would be a natural choice to take the space of all possible persistence diagrams, which we denoted by $\mathsf{Pers}$, as the output space. The first candidate for the utility function would be the function $u_D$ defined in (\ref{exp: utility DTM}). However, its output space, $\mathsf{Pers}$, contains persistence diagrams which have arbitrary many numbers of birth-death pairs; that is, to sample persistence diagrams from the whole $\mathsf{Pers}$ is inevitably an infinite-dimensional problem, which is technically difficult in computation. To bypass such an issue, we pre-specify a hyperparameter $M \in \mathbb{Z}_+$, a positive integer, and only take care of the space $\mathsf{Pers}_{M}$ of persistence diagrams having at most $M$ birth-death pairs at each dimension $q$. On each restricted space $\mathsf{Pers}_{M}$, for any given data set $D$, we re-define the function $u_D^{(q)}: \mathsf{Pers}_{M} \to \mathbb{R}$ as follows:
\[
u_D^{(q)}(\mathcal{P}) := -\db (\mathcal{P}, \mathcal{P}_q (D)).
\]
The utility function $u_D$ is also re-defined in the same way as in (\ref{exp: utility DTM}) by using the re-defined $u_D^{(q)}$s. Namely, the utility function $u_D : \big( \mathsf{Pers}_{M} \big)^{\ell + 1} \to \mathbb{R}$ is defined as 
\[ \label{exp: util DTM restrict} \tag{4.1}
u_D(\mathcal{P}_0, \dots, \mathcal{P}_{\ell}) := \sum_{q = 0}^{\ell} u_D^{(q)} (\mathcal{P}_q).
\]

Note that the upper and lower bounds established in Corollary \ref{cor: DTM sensitivity} are still valid for the utility $u_D$ defined in (\ref{exp: util DTM restrict}). 

Under the choice of the utility function $u_D$ in (\ref{exp: util DTM restrict}), the probability distribution from which privatized persistence diagrams are generated can be specified. Before describing the exponential mechanism, we introduce a discretized version of $\mathsf{Pers}_M$ that will make the analysis in Section 4.2 convenient. Note that each persistence diagram in $\mathsf{Pers}_M$ can be viewed as a family of at most $M$ points in the upper-left triangle $\bar{\mathcal{T}} := \{(x, y) \: : \: 0 \leq x \leq y \leq \mathrm{diam} E \}$. Instead of using $\bar{\mathcal{T}}$ directly, we discretize it with finitely many discrete points; for example, a set of equally-spaced finitely many points in $\bar{\mathcal{T}}$ can be a such discretization. By discretizing the set $\bar{\mathcal{T}}$ with $N^2$ discrete points, a discretization of $\mathsf{Pers}_M$ can be obtained; namely, the discretized version of $\mathsf{Pers}_M$ is the family of sets having at most $M$ points in the discretized version of $\bar{\mathcal{T}}$. Note that such a discretization of $\mathsf{Pers}_M$ has cardinality at most $N^{2M}$. For a given positive integer $N$, we define $\widetilde{\mathsf{Pers}}_{M,N}$ to be the discretization of $\mathsf{Pers}_M$ obtained by discretizing $\bar{\mathcal{T}}$ with $N^2$ equally spaced discrete points. Therefore, our exponential mechanism is indeed carried out on the space $\widetilde{\mathsf{Pers}}_{M,N}$. The space $\widetilde{\mathsf{Pers}}_{M,N}$ is the actual output space where the private persistence diagrams generated by the following mechanism live.

\begin{proposition} \label{prop: our density}
Let $\ell \geq 0$ be fixed and the utility $u_D$ defined in (\ref{exp: util DTM restrict}). Set $p(\cdot)$ to denote the probability density function characterized by the following expression:
\[ \tag{4.2}
\begin{aligned} \label{exp: our density} 
 p(\mathcal{P}_{\mathrm{DP}})  \propto \exp \bigg( \frac{\epsilon}{2 \Delta} u_D(\mathcal{P}_{\text{DP}}) \Bigg)  = \exp \bigg( -\frac{\epsilon}{2 \Delta} \sum_{q=0}^{\ell} \db \big( \mathcal{P}_q(D), \mathcal{P}_{q, \mathrm{DP}} \big) \bigg)
\end{aligned}
\]
with respect to the uniform distribution on the set $(\widetilde{\mathsf{Pers}}_{M,N})^{\ell + 1}$ as the base measure.
In (\ref{exp: our density}), $\Delta$ is defined as follows.
\[
\Delta := (\ell + 1) \frac{\mathrm{diam} E}{mn}
\]
and
$\mathcal{P}_{\mathrm{DP}} = (\mathcal{P}_{0, \mathrm{DP}}, \dots, \mathcal{P}_{\ell, \mathrm{DP}})$. Then, the exponential mechanism characterized by the density (\ref{exp: our density}) satisfies $\epsilon$-DP.
\end{proposition}

To generate a sample from the distribution (\ref{exp: our density}), we utilize the Metropolis-Hastings algorithm. The detailed procedure of the algorithm is summarized in Appendix \ref{appendix: supplement-alg}. 

\begin{remark}
    Note that the discretization is not necessary to establish Proposition \ref{prop: our density}, but is needed to derive the utility results in the following section, such as Proposition \ref{prop: privacy error upper} and Theorem \ref{thm: STAT error upper}. It is possible that this discretization can be removed from our analysis using more sophisticated techniques.
\end{remark}

\subsection{Analysis of privatized persistence diagram} \label{sec: ExpMech-error analysis}

Let $\ell \geq 0$ be determined. Recall that we have restricted the output space of our privatized persistence diagram in terms of $M$ points for each dimension, and that these fall on a discretized version of the continuous persistence diagram space.
To incorporate these limitations into our consideration for the error quantification, we define $\mathcal{P}_{\mathrm{OPT}}$ to be the closest persistence diagram from $\mathcal{P}(D) = (\mathcal{P}_0(D), \dots, \mathcal{P}_{\ell}(D)) $ that can be generated from the privacy algorithm. More precisely, for each $q$
\[
\mathcal{P}_{q, \mathrm{OPT}} := \underset{\mathcal{P} \in \mathsf{Pers}_{M}}{\mathrm{argmin}} \; \db (\mathcal{P}, \mathcal{P}_q(D)),
\]
where $\mathcal{P}$ ranges over all persistence diagrams having at most $M$ elements, and
\[
\mathcal{P}_{\mathrm{OPT}} := (\mathcal{P}_{0, \mathrm{OPT}}, \dots, \mathcal{P}_{\ell, \mathrm{OPT}}).
\]
Similarly, the counterpart of $\mathcal{P}_{\mathrm{OPT}}$ on the discrete space $\widetilde{\mathsf{Pers}}_{M,N}$ is defined as follows. For every $q \geq 0$,
\[
\widetilde{\mathcal{P}}_{q, \mathrm{OPT}} := \underset{\mathcal{P} \in \widetilde{\mathsf{Pers}}_{M,N}}{\mathrm{argmin}} \; \db \big(\mathcal{P}, \mathcal{P}_q(D) \big)
\]
and
\[
\widetilde{\mathcal{P}}_{\mathrm{OPT}} := (\widetilde{\mathcal{P}}_{0, \mathrm{OPT}}, \dots, \widetilde{\mathcal{P}}_{\ell, \mathrm{OPT}}).
\]
Moreover, for any pair of $(\ell + 1)$-tuples of persistence diagrams $ \mathcal{P} = (\mathcal{P}_0, \dots, \mathcal{P}_{\ell})$ and $\mathcal{P}^{'} = (\mathcal{P}_0^{'}, \dots, \mathcal{P}_{\ell}^{'})$, we define
\[
\db (\mathcal{P}, \mathcal{P}^{'}) := \sum_{q = 0}^{\ell} \db(\mathcal{P}_q, \mathcal{P}_q^{'}).
\]

In general, in the literature on the exponential mechanism, there have been broad analyses with regard to the error-minimizing value in the space covered by the exponential mechanism. For instance, see \cite{Dwork2014}. One key result concerning $\mathcal{P}_{q, \mathrm{OPT}}$ is summarized in Proposition \ref{exp: ExpMech-utility}. Consequently, the following estimate can be established. Recall that the discretized space $\widetilde{\mathsf{Pers}}_{M,N}$ has been obtained by discretizing the upper-left triangle $\bar{\mathcal{T}}$ with $N^2$ equally-spaced discrete points. 

\begin{proposition}
Let $\mathcal{P}_{\mathrm{OPT}}$ be defined in the above and $\mathcal{P}_{\mathrm{DP}}$ the private persistence diagram obtained from the exponential mechanism summarized in Section \ref{sec: ExpMech-algorithm}. Suppose that the upper-left triangle $\bar{\mathcal{T}}$ is discretized into $N^2$ equally spaced points. Then the following holds.
\[ \label{prop: privacy error upper}
\db (\mathcal{P}_{\mathrm{OPT}}, \mathcal{P}_{\mathrm{DP}}) = O_p \bigg( \frac{(\ell + 1)^2 M \log N}{n \epsilon} + \frac{1}{N} \bigg)
\]
In particular, if we take $N = n$, it holds that
\[
\db(\mathcal{P}_{\mathrm{OPT}}, \mathcal{P}_{\mathrm{DP}}) = O_p \bigg( \frac{(\ell + 1)^2 M \log n}{n \epsilon} \bigg) = \tilde{O}_p \bigg( \frac{(\ell + 1)^2 M}{n \epsilon} \bigg).
\]
\end{proposition}

\begin{remark}
In fact, the exponential mechanism itself only directly guarantees that privatized diagrams are concentrated at the optimal diagram in the discretized space. More specifically, we have
\[ \label{exp: privacy error discrete} \tag{4.3}
\db (\widetilde{\mathcal{P}}_{\mathrm{OPT}}, \mathcal{P}_{\mathrm{DP}}) = O_p \bigg( \frac{(\ell + 1)^2 M\log N}{n \epsilon}\bigg).
\]
On the other hand, as long as we employ fine enough discretization, it is trivial that the distance $\db (\mathcal{P}_{\mathrm{OPT}}, \widetilde{\mathcal{P}}_{\mathrm{OPT}})$ is negligible compared to the error (\ref{exp: privacy error discrete}). For instance, taking $N = N(n) = n$ ensures that such an approximation error has order $O_p(n^{-1})$ and the term $\log N$ in (\ref{exp: privacy error discrete}) only adds $\log n$ amount of error. This guarantees the result in Proposition \ref{prop: privacy error upper}. 
\end{remark}

To take advantage of the above result, we can estimate each of the two types of errors as follows.
\[ \label{exp: STAT error upper} \tag{4.4}
\db (\mathcal{P}_{\mathrm{DP}}, \mathcal{P}(\mu)) \leq \db (\mathcal{P}_{\mathrm{DP}}, \mathcal{P}_{\mathrm{OPT}}) + \db (\mathcal{P}_{\mathrm{OPT}}, \mathcal{P}(\mu))
\]
and
\[ \label{exp: Cs error upper} \tag{4.5}
\db (\mathcal{P}_{\mathrm{DP}}, \mathcal{P}(D)) \leq \db (\mathcal{P}_{\mathrm{DP}}, \mathcal{P}_{\mathrm{OPT}}) + \db (\mathcal{P}_{\mathrm{OPT}}, \mathcal{P}(D)).
\]
Hence, the remaining part is to estimate $\db (\mathcal{P}_{\mathrm{OPT}}, \mathcal{P} (\mu))$ and $\db (\mathcal{P}_{\mathrm{OPT}}, \mathcal{P}( D))$, respectively. 

Before stating our main theorem in this section, we would like to summarize the terminology that we use to call each of the error terms we are concerned with. First of all, we call $\db(\mathcal{P}(D), \mathcal{P}(\mu))$ the estimation error because $\mathcal{P}(D)$ can be viewed as a statistic estimating $\mathcal{P}(\mu)$. The term $\db (\mathcal{P}_{\mathrm{DP}}, \mathcal{P}_{\mathrm{OPT}})$ is called the privacy error, following the tradition in DP literature. On the other hand, we call the quantity $\db (\mathcal{P}_{\mathrm{OPT}}, \mathcal{P} (D))$ the approximation error because $\mathcal{P}_{\mathrm{OPT}}$ is the best approximation of $\mathcal{P}(D)$ in the space $\mathsf{Pers}_M$. In contrast with the previous two terms, a type of quantity of the form $\db (\mathcal{P}_{\mathrm{OPT}}, \mathcal{P} (\mu))$ has not been analyzed in the literature before to our knowledge. As we noted, $\mathcal{P}(\mu)$ can be regarded as a population parameter describing the probability measure $\mu$ generating the data $D$. Concerning this perspective, we call the quantity $\db (\mathcal{P}_{\mathrm{OPT}}, \mathcal{P}(\mu))$ the constrained estimation error and call the corresponding quantity $\db (\mathcal{P}_{\mathrm{DP}}, \mathcal{P} (\mu))$ the privacy-estimation error in order to allude that this quantity would be interpreted as the amount of error in estimating the population parameter $\mathcal{P} (\mu)$ by the privatized statistic $\mathcal{P}_{\mathrm{DP}}$. 


If we can choose $M$ large enough, both terms $\db (\mathcal{P}_{\mathrm{OPT}}, \mathcal{P}(\mu))$ and $\db (\mathcal{P}_{\mathrm{OPT}}, \mathcal{P}(D))$ can be estimated through the convergence of the empirical distribution on $D$ to the measure $\mu$ in terms of the Wasserstein distance $W_1$ (See Proposition \ref{prop: estimation error}). It turns out that both terms are bounded by $O_p((\ell + 1)n^{-1/d})$. Hence, by taking $M = M(n)$ to be a slowly increasing sequence we can achieve such a bound without degrading the privacy error obtained in Proposition \ref{prop: privacy error upper}.

\begin{theorem}[Upper bound for the privacy-estimation error] \label{thm: STAT error upper}
Set $M = M(n) = \log n$ and $N(n)=n$. Then, for all large enough $n$, the following estimate holds.
\[
\begin{aligned}
\db  (\mathcal{P}_{\mathrm{DP}}, \mathcal{P} (\mu) ) & \leq \db (\mathcal{P}_{\mathrm{DP}}, \mathcal{P}_{\mathrm{OPT}}) + \db (\mathcal{P}_{\mathrm{OPT}}, \mathcal{P} (\mu)) \\
& =  \tilde{O}_p \bigg(\frac{(\ell + 1)^2 }{n \epsilon} +  \frac{(\ell + 1)}{n^{1/d}} \bigg).
\end{aligned}
\]
\end{theorem}


It is natural to wonder how sharp the obtained upper bounds are. As for the population-estimation error (and the estimation error), unfortunately, 
it is inevitable to get the rate $n^{-1/d}$ so long as the argument relies on the Wasserstein convergence of the empirical measure on $D$ to the measure $\mu$. This means that if a tighter rate is possible,  it is required to use a different approach in order to examine the birth-and-death of homological features of the sub-level sets of the DTM more precisely. In the literature of TDA, there are some approaches that examined such features of \v{C}ech complexes by employing some toolkits from geometry. For example, see \cite{Bobrowski;Adler:2014}. Such approaches may hint how to analyze persistence diagrams of the DTM more precisely.

As for the privacy error, we argue that it is sharp up to constants and logarithmic factors. Recall that $\mathcal{P}_{\mathrm{OPT}}$ is defined to be the persistence diagram in the range of our privacy algorithm which has the smallest distance from $\mathcal{P}(D)$. This definition lets us surmise that the distance $\db (\mathcal{P}_{\mathrm{DP}}, \mathcal{P}_{\mathrm{OPT}})$ could be smaller than the distance $\db (\mathcal{P}_{\mathrm{DP}}, \mathcal{P}(D))$ in a considerable probability. This means that if we are able to find a lower bound of $\db (\mathcal{P}_{\mathrm{DP}}, \mathcal{P}(D))$ matching the upper bound of $\db (\mathcal{P}_{\mathrm{DP}}, \mathcal{P}_{\mathrm{OPT}})$, it underpins that our estimate could be sharp. In the following theorem, we prove that, under some mild conditions, there is no $\epsilon$-DP mechanism whose privacy error with respect to the persistence diagrams from $L^1$-DTM can be smaller than $1/(n \epsilon)$. For the following, we recall that $\mathcal{P}_0(D)$ denotes the $0$th persistence diagram obtained by the $L^1$-DTM to the empirical measure on a given data set, as defined before. 

\begin{theorem} \label{thm: CS error lower}
Suppose that $m < 1/2$. Let $n$ be a positive integer and $\mathcal{M}$ an arbitrary $\epsilon$-DP mechanism that produces a privatized persistence diagram $\mathcal{M}(D)$ of a data set $D$. Assume that $\epsilon$ satisfies $1/n \leq \epsilon \leq 1$. Then it is not possible for $\mathcal{M}$ to achieve $\db (\mathcal{P}_0 (D),\mathcal{M}(D)) = o_p \big( \frac{1}{n \epsilon} \big)$ for every database $D$ with $|D| = n$.
\end{theorem}


\section{Simulation Studies}
\label{sec: simulation}

In the following simulation studies and the real-world data analysis, we only consider the $0$th and the $1$st persistence diagrams; namely, the utility we use is given by taking $\ell = 1$, i.e., we set $u_D : \big( \mathsf{Pers}_M \big)^2 \to \mathbb{R}$ by
\[
u_D(\mathcal{P}_0, \mathcal{P}_1) := u_D^{(0)}(\mathcal{P}_0) + u_D^{(1)}(\mathcal{P}_1).
\]
The purpose of such restriction is only to present our algorithm succinctly; the algorithm can readily be extended to take the higher-dimensional features into consideration.

We produce the differentially private persistence diagrams and investigate the impact of the key hyper-parameters: the privacy budget $\epsilon$ and the sample size $n$, where the resolution of the DTM $m$ is set to 0.2. For the exponential mechanism, the default parameters are  $\epsilon = 1$, $m=0.2$, $n=4000$, and $M=5$. These hyper-parameters were chosen by preliminary simulations. To sample from the exponential mechanism, we use a Markov chain Monte Carlo algorithm, specified in Appendix~\ref{appendix: supplement-alg}. We choose the last iterate out of $T=10000$ Monte Carlo diagrams as the reporting privatized diagram to be used for analysis\footnote{The \texttt{R} code is available at \url{https://github.com/jwsohn612/DPTDA}.}.

The simulation is based on the example in Figure \ref{fig: two clusters}, which has two circles at the origin $(1.5,1.5)$ and $(-1.5,-1.5)$ whose radii are 1.5 and 1 respectively. Each circle consists of 200 observations of uniformly generated points along the boundary of the circle. There is one more point in the middle of the circles for (i-2) and (ii-2). When inducing the Vietoris-Rips diagrams, the maximum filtration scale is specified as 3. All analyses are based on 500 sampled data sets, and we apply the privacy mechanism for each replication.  



Figure~\ref{simul:replicates} illustrates the outputs of the exponential mechanism as $\ep$ and $n$ are varied. To reflect the variation of diagrams, we consider 500 replicated data sets. Our exponential algorithm is independently applied to each data set, and the algorithm reports the final diagram only. By repeating this procedure for all 500 replicates, we obtain the 500 reporting private diagrams. Each panel in Figure~\ref{simul:replicates} is drawn based on the 500 private diagrams that illustrate the shape of private diagrams' distribution. 
As expected, the variability in the privatized persistence diagrams becomes smaller as either $\epsilon$ or $n$ becomes larger.

The overall tendency in terms of the bottleneck distance is exhibited in Figure~\ref{simul:one_sens}. Note that the x-axis is written in the log scale. Gray areas in the panel show 95\% point-wise quantile intervals of the bottleneck distance between the non-private diagram and its private one. 
Figure~\ref{simul:one_sens} depicts that both $\ep$ and $n$ in the log scale have an approximately linear relationship to the log-bottleneck distance as shaded areas decently contain the red dotted lines e.g., $\log \db (\mathcal{P}(D),\cdot) \approx -\log n + c$ with some constant $c$.  These results heuristically support that $d_B(\mathcal(P)_{DP},\mathcal{P}(D))=\tilde O_p(1/(n\ep))$ (considering $\ell$ to be fixed).


\begin{figure}
\centering
    \psfrag{e1}{$\epsilon~=~0.1$}
    \psfrag{e2}{$\epsilon~=~1$}
    \psfrag{e3}{$\epsilon~=~10$}
    \psfrag{n1}{$n=250$}
    \psfrag{n2}{$n=1000$}
    \psfrag{n3}{$n=8000$}
\begin{subfigure}[b]{1.00\textwidth}
    \centering
    \includegraphics[width=0.9\textwidth]{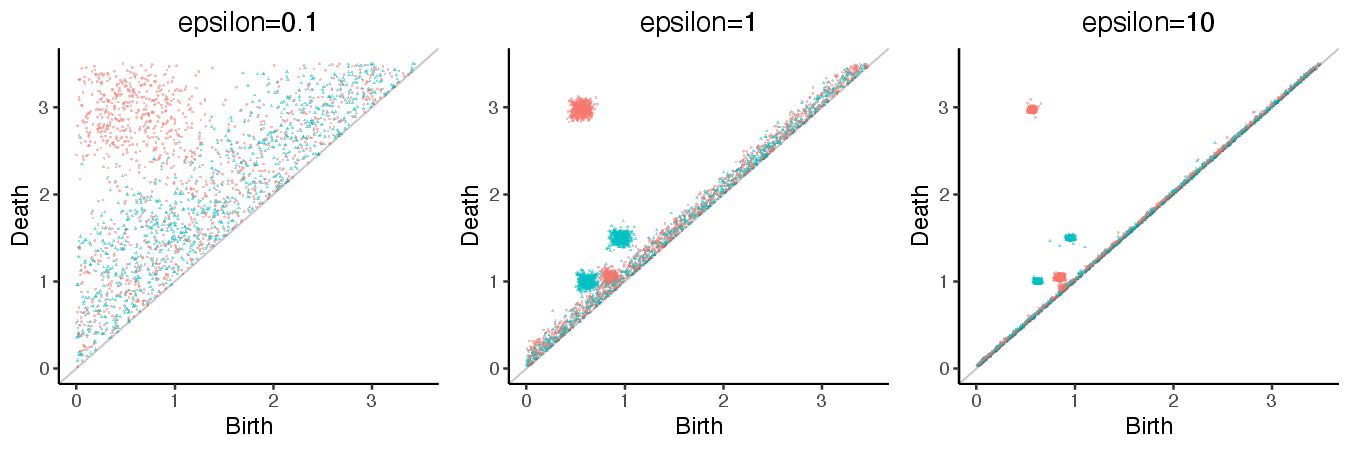}
    \caption{Privatized persistence diagrams of different $\epsilon$ where $n=4000$.}
\end{subfigure}
\hfill
\begin{subfigure}[b]{1.00\textwidth}
    \centering
    \includegraphics[width=0.9\textwidth]{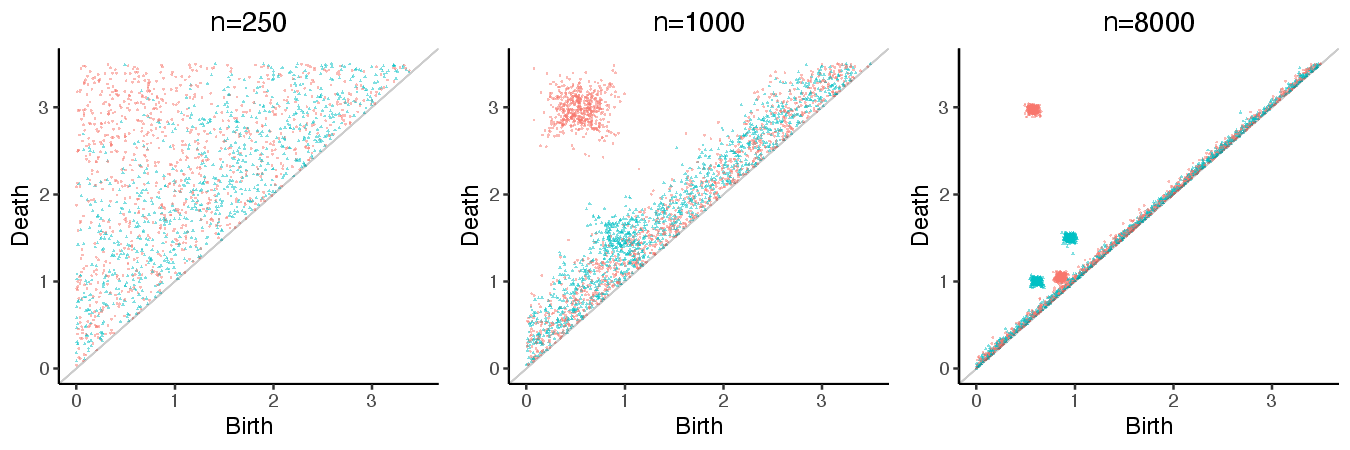}
    \caption{Privatized persistence diagrams of different $n$ where $\epsilon=1$.}
\end{subfigure} 
\caption{Privatized diagrams over 500 replicated data sets as described in Section \ref{sec: simulation}: $0$-dimensional connected components (orange) and $1$-dimensional loops (green). }
\label{simul:replicates}
\end{figure}

\begin{figure}
\centering
    \psfrag{leps}{$\log \epsilon$}
    \psfrag{lresl}{$\log m$}
    \psfrag{lsize}{$\log n$}
    \psfrag{med}{$W_{\infty}$}
  \includegraphics[width=0.75\textwidth]{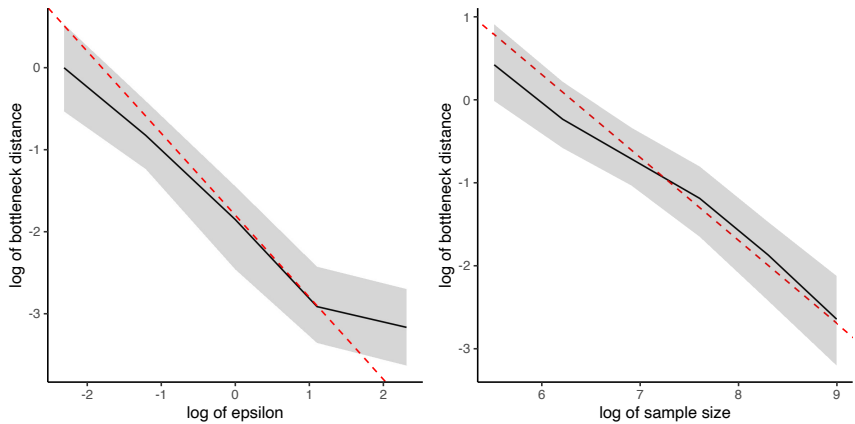}
      
\caption{The 95\% quantile intervals of $\db$ between $\mathcal{P}(D)$ and the corresponding private diagram over 500 replicates where $\epsilon$ (left) and $n$ (right) increase respectively. Red dotted lines captured overall in the shaded areas have -1 slopes.} 
\label{simul:one_sens}
\end{figure}

\section{Real Data Analysis}
\label{sec: real data}

In this section, we apply our mechanism to a real-world data set\footnote{Data is provided at \url{http://bertrand.michel.perso.math.cnrs.fr/Enseignements/TDA/Tuto-Part1.html}}, which tracks the locations of three people walking around within a building, recorded on smartphones. We are going to call those people Walker A, B, and C. The $3$-dimensional coordinates $(x, y, z)$ of the location of each person were measured $20000$ times over time so that the data set consists of $20000$ location vectors $(x, y, z)$ for each of Walker A, B, and C. 


We calculate the persistence diagram corresponding to each of the walkers and apply our mechanism in order to privatize it. We would like to remark that we are not concerned with the privacy of individual walkers, but we consider the privacy of an individual's time stamps when they change. If a particular walker's persistence diagram changes significantly in accordance with a change of location at a certain timestamp, then the location information could be exposed to a risk of privacy leakage. Our privacy mechanism retains the topological structure of each walker's travel while protecting the information associated with each timestamp.

To obtain a privatized diagram, we carry out $50000$ iterations in the MCMC procedure in our mechanism; the persistence diagram obtained at the last iteration is proposed as the reporting privatized diagram. We set the size of the sampling space $M = 5$, the resolution of the $L^1$-DTM $m = 0.05$, the privacy budget $\epsilon = 1$.

Figure \ref{fig: real data 2} depicts the results of comparing the $L^1$-DTM persistence diagram corresponding to Walker C and its privatized diagram. One can see the diagrams look quite similar. In fact, the bottleneck distances $\db (\mathcal{P}_0(D), \mathcal{P}_{0, \mathrm{DP}})$ and $\db (\mathcal{P}_1(D), \mathcal{P}_{1, \mathrm{DP}})$ are both $0.01$, which supports that our mechanism achieves high accuracy. Note that points near the diagonal are not considered significant, and do not substantially affect the bottleneck distance. In the right plot of Figure \ref{fig: real data 2}, we see that the bottleneck distance converges to a region $<.025$, and that the Markov chain seems to have converged after $\approx 5000$ iterations. 

\begin{figure}
    \includegraphics[width=\textwidth]{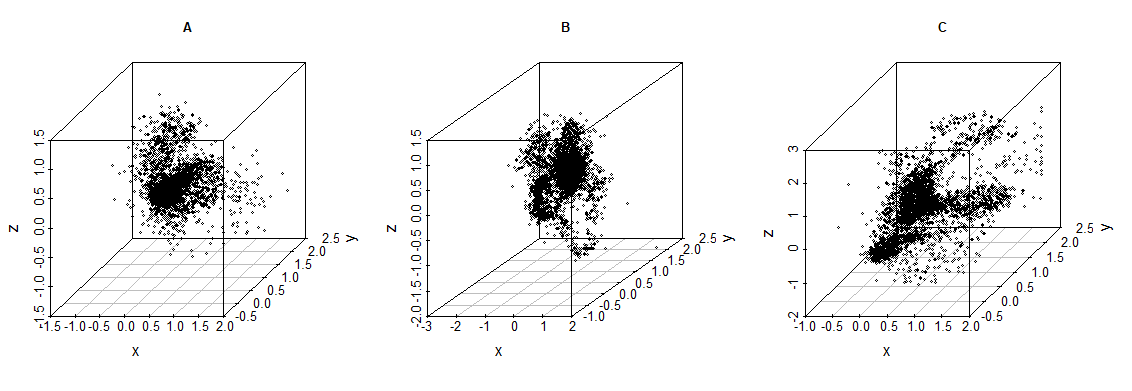}
    \caption{Scatter plots of the location information of Walker A,B, and C.}
    \label{fig: real data 1}
\end{figure}

\begin{figure}
    \includegraphics[width=\textwidth]{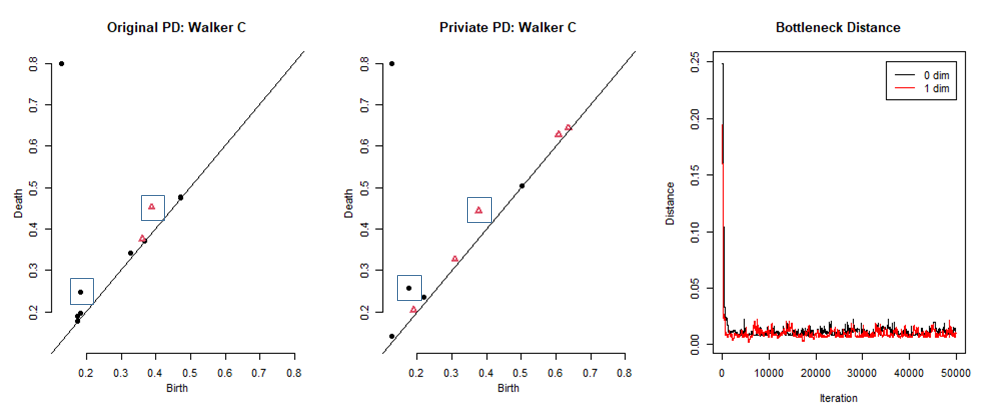}
    \caption{(Left) The true persistence diagram of Walker C; (Middle) a privatized persistence diagram of Walker C at the last iteration of a MCMC procedure, (Right) the bottleneck distances $\db (\mathcal{P}_0(D), \mathcal{P}_{0, \mathrm{DP}})$ (Black) and $\db(\mathcal{P}_1(D), \mathcal{P}_{1, \mathrm{DP}})$ (Red) of the true and a privatized diagram over MCMC iterations.  }
    \label{fig: real data 2}
\end{figure}

\section{Discussion}
\label{sec: discuss}
In this paper, we propose the first mechanism for producing a differentially private persistence diagram, while highlighting the role of outlier-robustness in the sensitivity analysis. Even though our results offer significant contributions to private TDA, as well as a general understanding of the robust TDA measures, there are still some important weaknesses of our work as well as directions for future work:

As noted in other papers (e.g., \citealp{minami2016differential,ganesh2020faster,seeman2021exact,awan2023privacy}), MCMC implementations of the exponential mechanism can incur additional privacy costs, which should be incorporated into the analysis for more rigorous studies. The proposed techniques in the above papers could be applied to our instance of the exponential mechanism to ensure that the MCMC approximation is taken into account for end-to-end privacy protection. 

While in this paper we recommended choosing $M$, the number of components in the persistence diagram, to be either a sufficiently large constant or an increasing function such as $M = \log n$, one could also consider $M$ to be an unknown quantity that also needs to be learned in a private manner. As one of the anonymous reviewers suggested, one may be able to develop a reversible jump MCMC algorithm to sample from the exponential mechanism in this case. Some challenges of this approach would be 1) developing a base measure over the infinite-dimensional space $\prod_{M=1}^\infty (\widetilde{\mathsf{Pers}}_{M,N})^{\ell+1}$, which ensures that the exponential mechanism results in a valid probability distribution, 2) determine a reversible jump rule that allows for conversion between the dimensions, and 3) perform a customized utility analysis of this new mechanism. We leave it to future researchers to consider this direction.

On the side of TDA, we would like to mention that some other outlier-robust TDA methods could be discussed for privacy protection. For instance, a kernel distance which was also discussed by \cite{Chazal2018} may be a good candidate. 

Besides TDA, the overall framework of how we propose a privacy mechanism can be applied to any other statistics that take their values in a metric space. A utility function concerned with a metric space-valued statistic can be defined similarly as we do with a persistence diagram and the bottleneck distance; this was already recognized in \cite{dwork2006calibrating}. However, the utility analysis for each different problem requires unique understanding of the specific structure and properties of the statistic and metric space at hand.

A theoretical limitation of our method is its scaling with the number of dimensions, denoted by $\ell + 1$, we consider. It is well known that the error in $\ep$-DP mechanisms typically scales linearly with the dimension, and our instance of the exponential mechanism is no different. Since this is a limitation of $\ep$-DP, it can only be addressed by using a different privacy framework. Future researchers may consider developing DP-TDA methods in the frameworks of approximate DP \citep{Dwork2014}, zero-concentrated DP \citep{bun2016concentrated} or Gaussian DP \citep{dong2022gaussian}, which often allow for the magnitude of the privacy noise to scale only in the square-root of the dimension. 

Another limitation of our work is that our utility analysis depends on an artificial discretization of the persistence diagram space. This limitation is caused by the use of Proposition \ref{prop: ExpMechUtility}, and could be potentially addressed by using more advanced techniques. 

Finally, \cite{dong;etal;20} proposed an alternative quantity to sensitivity for the exponential mechanism, which may improve the finite-sample performance.

\acks{The authors are thankful to the anonymous reviewers for their helpful comments that significantly improved the presentation of this manuscript. Taegyu Kang's research was partially supported by the AFOSR grant FA9550-22-0238 at Purdue University. Jordan Awan's research is supported in part by NSF grant SES-2150615. }


\bibliography{reference}


\appendix

\section{More on Persistent Homology} \label{appendix: TDA}

This part is devoted to providing more detailed background information about how to construct persistent homology and the corresponding persistence diagram. We start with the definition of simplicial complexes and simplicial homology, and then we introduce how to construct persistent homology.

\subsection{Simplicial homology} \label{appendix: TDA-1}

Let us start with the definition of simplicial complexes. Most of the contents of this subsection is based on \cite{Munkres1984}.

\begin{definition}[Simplicial complexes]
An (abstract) simplicial complex is a collection $K$ of finite non-empty sets, such that if $\sigma$ is an element of $K$, so is every non-empty subset of $\sigma$.
\end{definition}

Each element $\sigma$ of a simplicial complex $K$ is called a simplex of $K$. The dimension of the simplex $\sigma$ is defined to be $|\sigma|-1$, i.e., the number of elements in $\sigma$ minus one. When $\sigma$ is a $q$-dimensional simplex, we simply say that $\sigma$ is a $q$-simplex. The dimension $\dim K$ of the simplicial complex $K$ is defined to be the maximum dimension of the simplices in $K$, i.e., 
\[
\dim K := \max_{\sigma \in K} \dim \sigma.
\]
If the set $\{\dim \sigma \: : \: \sigma \in K\}$ is not bounded, set $\dim K = \infty$. Each non-empty subset of $\sigma$ is called a face of $\sigma$.

Let $K$ be a simplicial complex. For each simplex $\sigma = \{v_0, \dots, v_q\}$ in $K$, one can consider ordered tuples of vertices in $\sigma$. Namely, for every permutation $\alpha$ on $\{0, \dots, q\}$, there exists an ordered tuple $(v_{\alpha(0)}, \dots, v_{\alpha(1)})$. Such an ordered tuple is called a ordered simplex of $\sigma$. The collection of all ordered simplices of every simplex in $K$ is called the ordered simplicial complex of $K$, and denoted by $K_{\mathrm{ord}}$.

Let $K_{\mathrm{ord}}$ be an ordered simplicial complex of a simplicial complex $K$. Let $v^{\alpha} = (v_{\alpha(0)}, \dots, v_{\alpha(q)})$ and $v^{\beta} = (v_{\beta(0)}, \dots, v_{\beta(q)})$ be two ordered $q$-simplices of a common $q$-simplex $\sigma = \{v_0, \dots, v_q\}$. Declare $v^{\alpha} \sim v^{\beta}$ if and only if $\alpha$ and $\beta$ have the same sign, i.e., $\alpha$ and $\beta$ differ only by even numbers of transpositions. Notice that this relation defines an equivalence relation on the set of ordered simplices of every simplex $\sigma$. Let $[v_0, \dots, v_q]$ denote the equivalence class of the ordered simplex $(v_0, \dots, v_q)$. Such an equivalence class is called an oriented $q$-simplex. Namely, every $q$-simplex with $q \geq 1$ induces two oriented $q$-simplices. Let $K_{\mathrm{ori}}$ denote the set of all oriented simplices of every simplex in $K$. When there is no confusion, we use the symbol $\sigma$ to denote both a simplex and an oriented simplex.

For every natural number $q \geq 0$, set $K_{\mathrm{ori}}^q$ be the set of all oriented $q$-simplices of $K$. Define $C_q(K)$ be the set of all functions $c: K_{\mathrm{ori}}^q \to \mathbb{Z}$ satisfying the following.

\begin{itemize}

\item $c(\sigma) = - c(\sigma')$ if $\sigma$ and $\sigma'$ are opposite orientations of the same simplex.

\item $c(\sigma) = 0$ for all but finitely many oridented $q$-simplices $\sigma$, i.e., each $c$ is finitely supported.

\end{itemize}

One can equip a group structure on $C_q(K)$ by defining the group operation to be element-wise addition. Notice that $C_q(K)$ is an abelian group with that group structure. Moreover, it is straightforward that $C_q(K)$ is a free abelian group whose basis can be constructed by choosing exactly one oriented simplex for every simplex $\sigma$. One can represent every element $c$ in $C_q(K)$ by the finite $\mathbb{Z}$-linear combinations of oriented $q$-simplices of $K$, i.e., each $c$ can be written as
\[
c = \sum_{i = 1}^k n_i \sigma_i,
\]
where $k$ is finite, $n_i \in \mathbb{Z}$ and $\sigma_i \in K_{\mathrm{ori}}^q$ for all $1 \leq i \leq k$. Each function $c$ is called a $q$-chain of $K$ and $C_q(K)$ is called the group of oriented $q$-chains of $K$. We set $C_q(K) = 0$ if $q < 0$ or $q > \dim K$.

Now, we define the boundary operator of oriented chain complexes. 

\begin{definition}[Boundary operator]
Let $K$ be a simplicial complex. For every integer $q$, define
\[
\partial_q : K_{\mathrm{ori}, q} \to C_{q-1}(K)
\]
by assigning
\[
\partial_q : [v_0, \dots, v_q] \mapsto \sum_{i=0}^q (-1)^i [v_9, \dots, \hat{v}_i, \dots, v_q],
\]
where $[v_0, \dots, \hat{v}_i, \dots, v_q]$ is the $(q-1)$-oriented simplex obtained by deleting $v_i$ from $[v_0, \dots, v_q]$. Since $C_q(K)$ is a free abelian group, the map $\partial_q$ can be extended into a unique group homomorphism $\partial_q : C_q(K) \to C_{q-1}(K)$. This homomorphism is called the boundary operator.
\end{definition}

The key property of the boundary operator is the following:
\[
\partial_{q-1} \circ \partial_q = 0 \quad \text{ for every } q.
\]
In other words, the sequence $(C_q(K), \partial_q)_{q \in \mathbb{Z}}$ of abelian groups and group homomorphisms form a chain complex. This property can be rephrased as follows.
\[
\Image \partial_{q-1} \subseteq \Ker  \partial_q \quad \text{ for every } q,
\]
where $\Ker$ and $\Image$ mean the kernel and the image of a homomorphism, respectively. Since the sequence of groups of oriented chain complexes form a chain complex, it is possible to define the homology groups of it. Moreover, the kernel $\Ker \partial_q$ is usually written as $Z_q(K)$ and each of its elements is called a $q$-cycle; and, the image $\Image \partial_{q-1}$ is usually written as $B_q(K)$ and each of its elements is called a $q$-boundary.

\begin{definition}[Simplicial homology]
Let $K$ be a simplicial complex. For every integer $q$, the $q$th simplicial homology group is defined to be the following quotient group:
\[
Z_q(K) \big/ B_q(K) = \frac{\Ker (\partial_q : C_q(K) \to C_{q-1}(K))}{\Image (\partial_{q-1} : C_{q}(K) \to C_q(K))},
\]
and denoted by $H_q(K)$
\end{definition}

\begin{remark}
Instead of constructing $C_q(K)$ to be an abelian group, one can consider the free $R$-module on the same basis where $R$ is a commutative ring. The boundary operator can be defined by the same, and now it can be uniquely extended to be an $R$-module homomorphism $\partial_q : C_q(K) \to C_{q-1}(K)$. The resulting sequence $(C_q(K), \partial_q)$ of $R$-modules and $R$-module homomorphisms form a chain complex of $R$-modules, so the simplicial homology of it can be defined by the same way; in this case, each homology group $H_q(K)$ becomes an $R$-module as well. In such a case, we denote the $q$th simplicial homology module of $K$ by $H_q(K; R)$ and call it the $q$th simplicial homology of $K$ with coefficients in $R$.
\end{remark}

\subsection{Persistent homology} \label{appendix: TDA-2}

Let $\{K_r\}_{r \in \mathbb{R}}$ be a collection of simplicial complexes satisfying $K_{r_1} \subseteq K_{r_2}$ if $r_1 \leq r_2$. Such a collection is called a filtration of simplicial complexes (parametrized by $\mathbb{R}$). For each simplicial complex $K_r$ in the filtration, it is possible to construct the chain complex $(C_q(K_r), \partial_q)_{q \in \mathbb{Z}}$ and the corresponding homology groups $(H_q(K_r))_{q \in \mathbb{Z}}$. In addition, each inclusion map $\iota_{r_1, r_2} : K_{r_1} \to K_{r_2}$ ($r_1 \leq r_2$), induces a group homomorphism $C_q(K_{r_1}) \to C_q(K_{r_2})$, which is actually the inclusion map $C_q(K_{r_1}) \hookrightarrow C_q(K_{r_2})$ for every integer $q$; and, all such homomorphisms (inclusions) commute with boundary operators, i.e., each inclusion induces a chain map between chain complexes of oriented chains. Thus, each inclusion $\iota_{r_1, r_2}$ induces a homomorphism $\iota_{r_1, r_2}^q : H_q(K_{r_1}) \to H_q(K_{r_2})$ between homology groups for every $q$. This produces a collection $\{H_q(K_r)\}_{r \in \mathbb{R}}$ of simplicial homology groups accompanied with a group homomorphism $\iota_{r_1, r_2}^q : H_q(K_{r_1}) \to H_q(K_{r_2})$ for every $q$ and every pair $r_1 \leq r_2$. 

For each pair $r_1 \leq r_2$ and each $q$, the image of $\iota_{r_1, r_2}^q : H_q(K_{r_1}) \to H_q(K_{r_2})$, denoted by $\Image \iota_{r_1, r_2}^q$, is called the $q$th persistent homology group that persists from $r_1$ to $r_2$. The rank of the group $\Image \iota_{r_1, r_2}^q$ is called the $q$th persistent Betti number that persists from $r_1$ to $r_2$ and denoted by $\beta_{r_1, r_2}^q$. Intuitively, the Betti number $\beta_{r_1, r_2}^q$ represents the number of independent $q$-cycles that were born before the parameter $r_1$ and have not been dead until the parameter $r_2$ in the filtration. Furthermore, for each $q$-cycle in the filtration, it is possible to consider the parameter values at which the cycle shows up at first (birth) and disappears (death), respectively. 

Let $\sigma$ be a $q$-cycle that shows up in the filtration at some point, i.e., $\sigma$ is an element of $\Ker \partial_q (K_r)$ for some $r$. Then, it is possible to consider the birth and death times (parameter values) of it. By bringing together all birth-death pairs of all $q$-cycles in the filtration, one can form a multiset of points of the form $(b, d)$ with $b \leq d \leq \infty$. That multiset is called the $q$th persistence diagram of the filtration. The formal construction of the persistence diagram is involved with the structure theorem of finitely generated graded modules over a principal ideal domain, which is a theorem from abstract algebra. Please refer to \cite{Carlsson2009} and \cite{Edelsbrunner2008} for more formal and comprehensive discussion.

\subsection{Some ways to construct a filtration of simplicial complexes} \label{appendix: TDA-3}

Now, we introduce several ways to obtain a filtration of simplicial complexes that play a role in the main discussion of this paper. The contents of this subsection can be found in \cite{Edelsbrunner2009}.

Let $D = \{x_1, \dots, x_n\}$ be a finite subset of a metric space $(\mathcal{X}, d)$. For every non-negative real number $r \geq 0$, consider the ball $B(x_i; r) := \{y \in \mathcal{X} \: : \: d(y, x_i) < r)$ centered at each $i \in \{1, \dots, n\}$. The \v{C}ech complex $\breve{\mathcal{C}}(D; r)$ on $D$ with radius $r$ is the simplicial complex defined as follows. A subset $\sigma = \{x_{i_0},\dots, x_{i_q}\}$ of $D$ is a member of $\breve{\mathcal{C}}(D; r)$ if and only if $\cap_{j=0}^q B(x_{i_j}; r) \neq \emptyset$. Notice that $\breve{\mathcal{C}}(D; r_1) \subseteq \breve{\mathcal{C}}(D; r_2)$ for every pair $r_1 \leq r_2$. Hence, the collection $\{\breve{\mathcal{C}}(D; r)\}_{r \geq 0}$ of \v{C}ech complexes forms a filtration of simplicial complexes. 

There are several variants of the \v{C}ech complex. One of such variants is the Vietoris-Rips complex. The Vietoris-Rips complex $\mathrm{VR}(D; r)$ on $D$ with radius $r$ is defined as follows. A subset $\sigma = \{x_{i_0}, \dots, x_{i_q}\}$ of $D$ is a member of $\mathrm{VR}(D; r)$ if and only if $B(x_{i_{j_k}}; r) \cap B(x_{i_{j_l}}; r) \neq \emptyset$ for every $k, l \in \{0, \dots, q\}$, i.e., The balls $B(x_{i_0}; r), \dots B(x_{i_q}; r)$ pairwise intersect with one another. It is also obvious that $\mathrm{VR}(D; r_1) \subseteq \mathrm{VR}(D; r_2)$ whenever $r_1 \leq r_2$. Hence, the collection $\{\mathrm{VR}(D; r)\}_{r \geq 0}$ of Vietoris-Rips complexes forms a filtration of simplicial complexes. The following relationship between the \v{C}ech complex and the Vietoris-Rips complex indicates that, on a finite subset in an Euclidean space, the filtration of \v{C}ech complexes and that of Vietoris-Rips complexes have essentially the same information.

\begin{proposition} \label{prop: VR lemma}
Let $D = \{x_1, \dots, x_n\}$ be a finite subset of a Euclidean space equipped with the metric induced by the $\ell^2$-norm on it. Then, for every $r \geq 0$, the following holds.
\[
\breve{\mathcal{C}}(D; r) \subseteq \mathrm{VR}(D; r) \subseteq \breve{\mathcal{C}}(D; \sqrt{2}r).
\]
\end{proposition}

The last way of construction is obtained from a real-valued function defined on a metric space. Let $(\mathcal{X}, d)$ be triangulable a metric space and $f: \mathcal{X} \to \mathbb{R}$ be a real-valued continuous function. For each $r \in \mathbb{R}$, consider the sub-level set $L_r := f^{-1} \big( (-\infty, r] \big)$, which is a subset of $\mathcal{X}$. Notice that $L_{r_1} \subseteq L_{r_2}$ whenever $r_1 \leq r_2$. Moreover, since $\mathcal{X}$ is triangulable, all sub-level sets can be triangulized while respecting the inclusion relationships. Hence, the collection of such triangulizations of the collection $\{L_r\}_{r \in \mathbb{R}}$ of sub-level sets produces a filtration of simplicial complexes.

Before closing this section, we introduce a certain condition on a continuous function $f: \mathcal{X} \to \mathbb{R}$ that ensures that $f$ does not behave too wildly. 

\begin{definition}[Tame functions] \label{def: tame}
Let $(\mathcal{X}, d)$ be a triangulable metric space and $f: \mathcal{X} \to \mathbb{R}$ a real-valued continuous function. Set $X_r$ to be the triangulization of the sub-level set $f^{-1} \big( (-\infty, r] \big)$. Let $\iota_{r_1, r_2}^q : H_q(X_{r_1}) \to H_q(X_{r_2})$ be the group homomorphism induced by the inclusion map $\iota_{r_1, r_2} : X_{r_1} \to X_{r_2}$ for every pair $r_1 \leq r_2$. We call $r \in \mathbb{R}$ a homological critical value if there is no positive number $\epsilon > 0$ for which $\iota_{r - \epsilon, r+ \epsilon}^q$ is an isomorphism for each dimension $q$. The function $f$ is said to be tame if $f$ produces only finitely many homological critical values and all homology groups of all sub-level sets of it have finite rank.
\end{definition}

\section{Proofs of the Main Results} \label{appendix: proofs}

In this part, we present the detailed proofs of the theorems in Section \ref{sec: sensitivity} and Section \ref{sec: ExpMech}. Throughout this section, unless there is no further specification, the symbol $\parallel \cdot \parallel$ denotes the $\ell^2$-norm in the Euclidean space where the data points discussed in each proof live.

\subsection{Proofs of the Results in Section 3} \label{appendix: proofs-section3}

\begin{proof}[Proof of Lemma \ref{lem: Cech sensitivity lower}]
Fix $r > 0$ so that $d_{m-1} < r < d_m$, and let $G(D ; r)$ be the geometric graph with vertex set $\mathcal{X}$ and connecting threshold $r$; i.e., $D$ is the vertex set of $G(D; r)$, and each pair $\{x_i, x_j\}$ of vertices is an edge of it if $d(x_i, x_j) \leq r$. Since $(0, d_m)$ is an element of the diagram $\PCech(D)$, there are at least two connected components $\mathcal{Y}_1$ and $\mathcal{Y}_2$ in $G(D; r)$ which satisfy
\[ \label{proof thm4 1} \tag{A.1.1}
\min_{x_1 \in \mathcal{Y}_1, x_2 \in \mathcal{Y}_2} d(x_1, x_2) = 2d_m.
\]
Let $\mathcal{Y}_1$ and $\mathcal{Y}_2$ be such connected components, and let $x_1 \in \mathcal{Y}_1$ and $x_2 \in \mathcal{Y}_2$ be the points attaining the minimum, i.e., $d(x_1, x_2) = d_m$. Set $D'$ to be the set obtained by adding one point, say $z$, at the mid-point of $x_1$ and $x_2$. It is obvious that the death time of $\mathcal{Y}_1$ (or equivalently, $\mathcal{Y}_2$) is cut in half. Notice that the death times of the other connected components in the filtration of $D$ cannot be bigger by adding the point $z$. Thus, we can write
\[
\PCech(D') = \big\{(0, d_1^{'}), \dots, (0, d_t^{'}), (0, \infty) \} \cup \{(0, d_m/2)\}
\]
with $d_j^{'} \leq d_{m-1}$ for all $j = 1, \dots, t$. Here, the element $(0, d_m/2)$ has multiplicity at least $2$.  

To calculate the bottleneck distance between $\PCech(D)$ and $\PCech(D')$ we have to consider all possible bijections between $\PCech(D)$ and $\PCech(D')$. All such bijections can be classified into three categories. First, $(0, d_m) \in \PCech(D)$ is associated with element $(0, d_j^{'}) \in \PCech(D')$ for some $j \in \{1, \dots, t\}$. Second, $(0, d_m) \in \PCech(D)$ is associated with $(0, d_m/2) \in \PCech(D')$. Third, $(0, d_m) \in \PCech(D)$ is associated with a point in the diagonal line. In the first case, the possible minimum distance concerning $(0, d_m)$ cannot be smaller than $\delta$. In the second case, the distance between $(0, d_m)$ and $(0, d_m/2)$ is obviously $d_m/2$. In the last case, the distance between $(0, d_m)$ and the diagonal line is $d_m/ \sqrt{2}$. Since the bottleneck distance is defined by taking the minimum over all such bijections, the desired result follows. 
\end{proof}

\begin{proof}[Proof of Theorem \ref{thm: Cech sensitivity lower}]
Suppose that $n$ is even. Let $a$ and $b$ be two points in the set $E$ with $|a - b| = \mathrm{diam} E$, and $D$ consist of $n/2$ copies of $a$ and $n/2$ copies of $b$. Let $D'$ be obtained by moving one of $a$s to the mid-point of $a$ and $b$, say $c$. Then, it is obvious that
\[
\PCech(D) = \left\{ \left(0, \mathrm{diam} E/2\right), \left(0, \infty\right) \right\}
\]
and
\[
\PCech(D') = \left\{ \left(0, \mathrm{diam} E/4\right), \left(0, \mathrm{diam} E/4\right), \left(0, \infty\right) \right\}.
\]
This proves that the bottleneck distance between these two diagrams is lower bounded by $\mathrm{diam} E /4$, which implies the desired result. When $n$ is odd, one can take $D$ to have $(n-1)/2$ copies of $a$ and $(n+1)/2$ copies of $b$, and the result does not change.

As for the second result, the proposed upper bound can be established by applying the reverse triangle inequality. To establish the lower bound, notice that for any pair of sets $D$ and $D'$, 
\[
\begin{aligned}
\sup_{\mathcal{P}} \left|v_D(\mathcal{P}) - v_{D'}(\mathcal{P}) \right| & \geq \left|v_D\left(\PCech (D)\right) - v_{D'}\left(\PCech (D)\right) \right| \\
& = \db \left( \PCech(D'), \PCech(D) \right)
\end{aligned}
\]
The supremum of the last expression over all adjacent pairs $D$ and $D'$ is lower bounded by $\mathrm{diam} E/4$ as a consequence of the first result. This completes the proof.
\end{proof}

\begin{proof}[Proof of Theorem \ref{thm: DTM sensitivity upper}]
Let $\delta_D^{(p)}$ and $\delta_{D'}^{(p)}$ be $L^p$-DTM to the empirical distributions of $D$ and $D'$, respectively. By the stability theorem (\ref{exp: bottlneck stability}) and the Wasserstein stability (\ref{exp: Wasserstein stability}) of the DTM, we have
\[ \label{proof thm8 2} \tag{A.1.2}
\db \left(\PDTMq(D), \PDTMq(D')\right) \leq \left\lVert \delta_D^{(p)} - \delta_{D'}^{(p)} \right\rVert{\infty} \leq \frac{1}{m^{1/p}} W_p\left(\hat{\mu}_D, \hat{\mu}_{D'}\right),
\]
where $\hat{\mu}_D$ and $\hat{\mu}_{D'}$ represent the empirical distributions on $D$ and $D'$, respectively. We are going to establish a upper bound of the right-hand side of the inequality (\ref{proof thm8 2}).

Assume that $H(D, D') = 1$. Let $x$ be the element that is in $D$ but not in $D'$, and $z$ be the element that is in $D'$ but not in $D$. Let $\pi$ be the coupling of $\hat{\mu}_D$ and $\hat{\mu}_{D'}$ defined as follows: For every $y \in D$, set
\[
\pi(y, y) = \frac{1}{n},
\]
and
\[
\pi(x, z) = \frac{1}{n}.
\]
It is straightforward to verify that $\pi$ is indeed a coupling of $\hat{\mu}_D$ and $\hat{\mu}_{D'}$. With this $\pi$ we have
\[
\begin{aligned}
\int_{(z_1, z_2) \in \mathbb{R}^d \times \mathbb{R}^d} \parallel z_1 - z_2 \parallel^p d \pi(z_1, z_2)  =  \parallel x - z \parallel^p \frac{1}{n}  \leq \big( \mathrm{diam} E \big)^p \frac{1}{n}
\end{aligned}
\]
By the definition of the Wasserstein distance $W_p$, we have
\[
\begin{aligned}
W_p(\hat{\mu}_D,\hat{\mu}_{D'})^p & = \inf_{\nu} \int_{(z_1, z_2) \in \mathbb{R}^d \times \mathbb{R}^d} \parallel z_1 - z_2 \parallel^p d\nu(z_1, z_2) \\
& \leq \int_{(z_1, z_2) \in \mathbb{R}^d \times \mathbb{R}^d} \parallel z_1 - z_2 \parallel^p d\pi(z_1, z_2)
\end{aligned}
\]
where $\nu$ ranges over all couplings of $P_n$ and $P_n^{'}$. Therefore, we obtain the following:
\[
W_p(\hat{\mu}_D,  \hat{\mu}_{D'}) \leq \frac{\mathrm{diam} E}{n^{1/p}},
\]
which implies the desired result.
\end{proof}

\begin{proof}[Proof of Proposition \ref{prop: DTM sensitivity lower}]
Let $D$ be a data set whose points are split into $50\%$ and $50\%$ at two ends $a, b$ of $E$ respectively. More specifically, $\parallel a - b \parallel = \diam E$ and the set $D$ contains $2/n$ copies of $a$ and $2/n$ copies of $b$. Let $c$ be the mid-point of $a$ and $b$; that is $\parallel a - c \parallel = \parallel b - c\parallel = \diam E/2$. Construct $D'$ by moving one $a$ in $D$ to $c$; namely, $D'$ has $n/2 - 1$ copies of $a$, $n/2$ copies of $b$, and one $c$. Let $\delta_D$ be the $L^1$-DTM to the empirical distribution on $D$ with resolution $m$ and $\delta_{D'}$ likewise. Then, we have
\[
\begin{aligned}
\delta_D(x) = \begin{cases}
0 & \text{ if } x = a, \\
0 & \text{ if } x = b, \\
\diam E/2 & \text{ if } x = c.
\end{cases}
\end{aligned}
\]
On the other hand,
\[
\begin{aligned}
\delta_{D'}(x) = \begin{cases}
0 & \text{ if } x = a, \\
0 & \text{ if } x = b, \\
\frac{k-1}{k} \frac{\diam E}{2} & \text{ if } x = c.
\end{cases}
\end{aligned}
\]
Recall that $k = \lceil mn \rceil$. Notice that any point $x$ on the line segment $\overline{ab}$ satisfies $\delta_D(x) \leq \delta_D(c)$ and $\delta_{D'}(x) \leq \delta_{D'}(c)$. Hence, the $0$th persistence diagram $\PDTMone(D)$ of $D$ is obtained as follows:
\[
\PDTMone(D) = \{ (0, \diam E/2), (0, \infty) \}.
\]
Similarly, $\PDTMone(D')$ is obtained as follows:
\[
\PDTMone(D') = \bigg\{ ( 0, (k-1)\diam E/(2k)), (0, \infty) \bigg\}.
\]
The bottleneck distance between the two diagrams above is calculated as follows:
\[
\begin{aligned}
\db (\PDTMone(D), \PDTMone(D')) = \frac{\diam E}{2} - \bigg( \frac{k-1}{k} \bigg) \frac{\diam E}{2} & = \frac{1}{k} \frac{\diam E}{2}  = \frac{\diam E}{2\lceil mn \rceil }.
\end{aligned}
\]
\end{proof}

\begin{proof}[Proof of Corollary \ref{cor: DTM sensitivity}]
The upper bound is obtained by applying the reverse triangle inequality. 

As for the lower bound, notice that
\[
\begin{aligned}
\sup_{\mathcal{P}} |u_D(\mathcal{P}) - u_{D'}(\mathcal{P})| &\geq |u_D(\mathcal{P}^{\mathrm{DTM}_1}(D)) - u_{D'}(\mathcal{P}^{\mathrm{DTM}_1}(D)) |\\ & = \db \big( \mathcal{P}^{\mathrm{DTM}_1}(D'), \mathcal{P}^{\mathrm{DTM}_1}(D) \big)  \geq \db \big( \PDTMone(D'), \PDTMone(D) \big).
\end{aligned}
\]
The last expression is bounded below by $\mathrm{diam} E/(  2 \lceil mn \rceil)$ as a result of Proposition \ref{prop: DTM sensitivity lower}. 
\end{proof}

\subsection{Proofs of the Results in Section 4} \label{appendix: proofs-section4}

\begin{proof}[Proof of Proposition \ref{prop: privacy error upper}]
Notice that the inequality (\ref{exp: ExpMech-utility}) gives
\[
\mathbb{P} \bigg[ \ | u_D(\mathcal{P}_{\mathrm{DP}}) - u_D(\widetilde{\mathcal{P}}_{\mathrm{OPT}})| \geq \frac{2 \Delta}{\epsilon} (\log |\widetilde{\mathsf{Pers}}_{M,N}^{(\ell + 1)}| + t) \bigg] \leq e^{-t},
\]
for every $t \geq 0$. The reverse triangle inequality yields
\[
\db(\mathcal{P}_{\mathrm{DP}}, \widetilde{\mathcal{P}}_{\mathrm{OPT}}) \geq |u_D(\mathcal{P}_{\mathrm{DP}}) - u_D(\widetilde{\mathcal{P}}_{\mathrm{OPT}})|.
\]
Combining those two yields
\[
\db (\mathcal{P}_{\mathrm{DP}}, \widetilde{\mathcal{P}}_{\mathrm{OPT}}) = O_p \left( \frac{\Delta}{\epsilon} (\ell + 1) \log \left|\widetilde{\mathsf{Pers}}_{M,N}\right| \right). 
\]
Recall that $\Delta$ is chosen to be $\Delta = (\ell + 1) \mathrm{diam} E/(mn)$ and $|\widetilde{\mathsf{Pers}}_{M,N}| = N^M$. Thus, we have
\[
\db (\mathcal{P}_{\mathrm{DP}}, \widetilde{\mathcal{P}}_{\mathrm{OPT}}) = O_p \bigg( \frac{(\ell + 1)^2 M \log N}{n \epsilon} \bigg).
\]

Now, recall that the upper-left triangle $\bar{\mathcal{T}}$ is discretized by $N^2 = N^2(n) = n^2$ equally-spaced points; the length of each spacing is bounded by $C \mathrm{diam} E/ n$ for some constant $C$ that only depends on the chosen Euclidean distance. Hence, with all large enough $n$, the error in approximating $\mathcal{P}_{\mathrm{OPT}}$ by $\widetilde{\mathcal{P}}_{\mathrm{OPT}}$ satisfies
\[
\db (\mathcal{P}_{\mathrm{OPT}}, \widetilde{\mathcal{P}}_{\mathrm{OPT}}) = O_p (n^{-1}),
\]
which completes the proof.
\end{proof}

Theorem \ref{thm: STAT error upper} can be proved by establishing the following result.

\begin{proposition} \label{prop: estimation error}
Let $M = M(n)$ be an non-decreasing sequence of positive integers satisfying $M(n) \geq |\mathcal{P}_q(\mu)|$ for all large enough $n$. Then, for every $q \geq 0$, we have
\[
\db (\mathcal{P}_{q, \mathrm{OPT}}, \mathcal{P}_q(\mu)) = O_p(n^{-1/d}).
\]
Moreover, we also have
\[
\db (\mathcal{P}_{q, \mathrm{OPT}}, \mathcal{P}_q(D)) = O_p(n^{-1/d}).
\]
\end{proposition}

\begin{proof}[Proof of Proposition \ref{prop: estimation error}]
With large enough $n$, we can assume that $M \geq |\mathcal{P}_q (\mu)|$. In other words, $\mathcal{P}_q(\mu)$ belongs to the space of persistence diagrams having at most $M$ elements. Hence, by the definition of $\mathcal{P}_{q, \mathrm{OPT}}$,
\[
\db (\mathcal{P}_q(D), \mathcal{P}_{q, \mathrm{OPT}}) \leq \db (\mathcal{P}_q(D), \mathcal{P}_q(\mu)).
\]

As for $\db (\mathcal{P}_q (\mu), \mathcal{P}_{q, \mathrm{OPT}})$, the triangle inequality gives
\[
\begin{aligned}
\db (\mathcal{P}_q (\mu), \mathcal{P}_{q, \mathrm{OPT}}) & \leq \db (\mathcal{P}_q (\mu), \mathcal{P}_q(D)) + \db (\mathcal{P}_q(D), \mathcal{P}_{q, \mathrm{OPT}}) \\
& \leq 2 \db ( \mathcal{P}_q(\mu), \mathcal{P}_q(D)).
\end{aligned}
\]

According to Theorem 2 in \cite{Fournier2015} along with the stability theorem (\ref{exp: bottlneck stability}) of the bottleneck distance and the Wasserstein stability (\ref{exp: Wasserstein stability}) of the DTM, it is straightforward to deduce that
\[
\db (\mathcal{P}_q(\mu), \mathcal{P}_q(D)) = O_p(n^{-1/d})
\]
for every $q \geq 0$.
\end{proof}

\begin{proof}[Proof of Theorem \ref{thm: STAT error upper}]
According to Proposition \ref{prop: estimation error}, we have
\[
\db (\mathcal{P}_{q, \mathrm{OPT}}, \mathcal{P}_q(\mu)) = O_p(n^{-1/d})
\]
for every $q \geq 0$. This estimate, together with the estimate given in Proposition \ref{prop: privacy error upper}, establishes the desired result.
\end{proof}

The proof of Theorem \ref{thm: CS error lower} is achieved by establishing the following three lemmas. The first lemma is rather technical.

\begin{lemma} \label{lemma: proof of DTM lower 1}
Set $\mathcal{P} = \PDTMone$. Assume that the resolution $m$ of the DTM is chosen to satisfy $m \leq 1/2$. For any pair of positive integers $K$ and $n$ with $1 \leq K \leq n$, there exists a pair of data sets $X_n$ and $Y_n$ satisfying $|X_n| = |Y_n| = n$, $H(X_n, Y_n) = K$, and $\db (\mathcal{P}(X_n), \mathcal{P}(Y_n)) \geq \frac{C K}{n}$ for some constant $C$ independent of $K$ and $n$, where $H(X_n, Y_n)$ denotes the Hamming distance between $X_n$ and $Y_n$.
\end{lemma}

\begin{proof}

Recall that $k = \lceil mn \rceil$. The whole situation will be broken down into three cases: (i) $1 \leq K \leq \min\{k, n/2 - k\}$, (ii) $\min\{k, n/2 - k\} < K < \max\{k, n/2 - k\}$, and (iii) $K \geq \max\{k, n/2 - k\}$.

First, let us assume that $1 \leq K \leq \min\{k, n/2 - k\}$. Choose two points $a$ and $b$ satisfying $\parallel a - b \parallel = \diam E$; for instance, in the case of $E = [0, 1]^d$, one may choose $a = (0, \dots, 0)$ and $b = (1, \dots, 1)$. Choose the data set $X_n$ that consists of $n/2$ copies of $a$ and $n/2$ copies of $b$ (If $n$ is odd, take $(n-1)/2$ copies of $a$ and $(n+1)/2$ copies of $b$; the results will be the same). On the other hand, choose the data set $Y_n$ constructed by moving $K$ copies of $a$ to the mid-point of $a$ and $b$, say $c$, i.e., as multisets, $X_n$ and $Y_n$ can be expressed as follows:
\[
X_n = \big\{\left(a, n/2\right), \left( b, n/2 \right) \big\} \text{ and } Y_n = \big\{ \left(a, n/2 - K\right), \left(c, K\right), \left( b, n/2 \right) \big\}.
\]
Since $M \leqslant n/2 - k$, the point $a$ still has more than $k$ numbers of points in the data set $Y_n$. Thus, we have
\[
\begin{aligned}
\delta_{Y_n}(x) = \begin{cases}
0 & \text{ if } x = a, \\
\frac{k-K}{k} \frac{\diam E}{2} & \text{ if } x = c, \\
0 & \text{ if } x = b.
\end{cases}
\end{aligned}
\]
Let $x(t)$ be the point in the line segment $\overline{ac}$ that divides $\overline{ac}$ into the ratio $t : (1-t)$ with $t \in [0, 1]$. Then, we have
\[
\begin{aligned}
\delta_{Y_n}(x(t)) = \begin{cases}
t \frac{\diam E}{2} & \text{ if } 0 \leq t \leq 1/2, \\
\frac{(k-K)t \diam E / 2 + K(1-t) \diam E/2}{k} & \text{ if } 1/2 < t \leq 1.
\end{cases}
\end{aligned}
\]

Now, let us further decompose the situation into two cases: (i-1) $(k - 2K) \geq 0 \Leftrightarrow K \leq k/2$ and (i-2) $(k - 2K) < 0 \Leftrightarrow K > k/2$. In the case (i-1), $\delta_{Y_n}(x(t))$ is increasing in $t$. Hence, $\mathcal{P}(Y_n)$ is obtained to be
\[
\mathcal{P}(Y_n) = \{(0, \infty), (0, \delta_{Y_n}(c)) \}.
\]
Notice that $\mathcal{P}(X_n)$ is obtained to be
\[
\mathcal{P}(X_n) = \{(0, \infty), (0, \diam E/2)\}.
\]
Therefore, 
\[
\begin{aligned}
\db (\mathcal{P}(X_n), \mathcal{P}(Y_n)) = \frac{\diam E}{2} - \delta_{Y_n}(c) & = \frac{\diam E}{2} - \left( \frac{k-K}{k} \right) \frac{\diam E}{2} \\
& = \frac{K}{k} \frac{\diam E}{2} \\
& = \frac{\diam E}{2 m} \frac{K}{n}.
\end{aligned}
\]

In the case (i-2), $\delta_{Y_n}((x(t))$ decreases from $t = 1/2$ to $t = 1$. Thus, $\mathcal{P}(Y_n)$ is given to be
\[
\mathcal{P}(Y_n) = \{(0, \infty), (\delta_{Y_n}(c), \diam E/4), (0, \diam E / 4)\}. 
\]
the bottleneck distance between $\mathcal{P}(X_n)$ and $\mathcal{P}(Y_n)$ can be derived by comparing the two different scenarios. First case corresponds $(0, \diam /4)$ in $\mathcal{P}(Y_n)$ to $(0, \diam E /2)$ in $\mathcal{P}(X_n)$. The distance obtained from this case must be greater than or equal to $\diam E/4$. The other case corresponds $(0, \diam E/2)$ in $\mathcal{P}(X_n)$ to $(\delta_{Y_n}(c), \diam E/4)$ in $\mathcal{P}(Y_n)$. Consequently, $(0, \diam E/4)$ in $\mathcal{P}(Y_n)$ must correspond to a point in the diagonal. Thus, the distance obtained in this case must be greater than or equal to $\diam E / (4 \sqrt{2})$. Therefore,
\[
\db (\mathcal{P}(X_n), \mathcal{P}(Y_n)) \geq \frac{\diam E}{4 \sqrt{2} }.
\]

Now, let us turn our attention to the case (ii), which assumes that $\min\{k , n/2 - k\} < K < \max\{k, n/2 - k\}$. First, consider the case $k < n/2 - k$, so that $k < K < n/2 - k$. In this case, both $a$ and $c$ have at least $k$ points, so
\[
\delta_{Y_n}(x) = 0 \text{ for all } x = a, b, \text{and } c.
\]
The above result gives us 
\[
\mathcal{P}(Y_n) = \{(0, \infty), (0, \diam/4), (0, \diam E/4)\}.
\]
Thus,
\[
\db (\mathcal{P}(X_n), \mathcal{P}(Y_n)) = \frac{\diam E}{4 \sqrt{2}}.
\]

Second, consider the case $k > n/2 - k$, so that $n/2 - k < K < k$. In this case both $a$ and $c$ have less than $k$ points. Thus,
\[
\begin{aligned}
\delta_{Y_n}(x) = \begin{cases} \frac{k - n/2 + K}{k} \frac{\diam E}{2} & \text{ if } x = a, \\
\frac{k - K}{k} \frac{\diam E}{2} & \text{ if } x = c, \\
0 & \text{ if } x = b
\end{cases}.
\end{aligned}
\]
Using the similar argument we utilized in the case (i), it is possible to demonstrate that the desired result is true in this case too. 

Finally, let us consider the case (iii) where $K \geq \max\{k, n/2 - k\}$. In this case, $\mathcal{P}(Y_n)$ has at least one element $(0, \diam/4)$. Hence its bottleneck distance from $\mathcal{P}(X_n)$ is always greater than or equal to $\diam E/4$. This completes the proof of the lemma. 

\end{proof}

The next two lemmas address the concept of DP in terms of a hypothesis testing framework. Lemma \ref{lemma: proof of DTM lower 2} is a well-known folklore result in the DP literature. It can be easily derived using the $f$-DP framework \citep{dong2022gaussian}. We give a direct proof that does not require using $f$-DP.

\begin{lemma} \label{lemma: proof of DTM lower 2}
Let $X$ and $X'$ be adjacent data sets, and $\mathcal{M}$ be any $\epsilon$-DP mechanism. For a hypothesis test $H_0: {\cal M}(X)$ versus $H_1: {\cal M}(X')$, 
\[
\emph{Type I error} + \emph{Type II error} \geq \frac{2}{1+e^{\ep}}.
\]
\end{lemma}

\begin{proof}
Call $\mathcal Y$ the probability space that $\mathcal M(X)$ lives in. Call $\mu_X$  the probability measures on $\mathcal Y$ for $\mathcal M(X)$.  By \citet[Proposition 2.3]{awan2019benefits}, there exists a base measure $\nu$, which dominates $\mu_X$ for all databases $X$. Call $f_X$ the density of $\mu_X$ with respect to $\nu$, which by \citet[Proposition 2.3]{awan2019benefits} satisfies $f_X\leq e^\ep f_{X'}$ almost everywhere $\nu$, for adjacent databases $X$ and $X'$. 

Let $X$ and $X'$ be adjacent databases, and let $\phi:\mathcal Y\rightarrow [0,1]$ be a test. Then the type I and type II errors are $\mathrm{I}=\EE \phi(\mathcal M(X))$ and $\mathrm{II}=1-\EE\phi(\mathcal M(X'))$, respectively. Then 
\[
\begin{aligned}
\mathrm{I} = \EE\phi(\mathcal M(X)) = \int \phi(t) f_X(t)  d\nu & \geq e^{-\ep} \int \phi(t) f_{X'}(t)  d\nu \\
& = e^{-\ep}\EE\phi(\mathcal M(X'))\\
& = e^{-\ep}(1- \mathrm{II}).
\end{aligned}
\]
Repeating the argument using $\phi'=1-\phi$ and swapping the roles of $X$ and $X'$, we get 
\[
\mathrm{II} \geq e^{-\ep}(1-\mathrm{I}).
\]
Then,
\[
\mathrm{I} + \mathrm{II} \geq e^{-\ep}[2-(\mathrm{I} + \mathrm{II})],
\]
which implies that $\mathrm{I} + \mathrm{II} \geq \frac{2}{1+e^{\ep}}$.

\end{proof}

\begin{lemma} \label{lemma: proof of DTM lower 3}
Let $(\epsilon_n)_{n=1}^{\infty}$ be a sequence of positive numbers satisfying $1/n \leq \epsilon_n \leq 1$ for every $n$. Set $K_n = \lfloor 1/\epsilon_n \rfloor$. For each $n$, For a given sequence of positive numbers $(\Delta_n)_{n=1}^{\infty}$, let $\{(X_n,Y_n)\}_{n=1}^{\infty}$ be a sequence of finite data sets satisfying, for each $n$, $H(X_n,Y_n) = K_n$ and $\db (\mathcal{P}(X_n),\mathcal{P}(Y_n)) \geq K_n \Delta_n$. Here, $H$ denotes the Hamming distance between sets and $\mathcal{P}$ means an arbitrary persistence diagram. Then for any $\epsilon(n)$-DP mechanism $\mathcal{M}$ that produces a privatized persistence diagram, it is not possible for both $\db (P_{X_n},\mathcal{M}(X_n))=o_p(\Delta_n/\epsilon_n)$ and $\db (P_{Y_n},\mathcal{M}(Y_n))=o_p(\Delta_n/\epsilon_n)$. 
\end{lemma}

\begin{proof}
For simplicity of notation, we will suppress the dependence of $X$, $Y$, $\ep$, $\Delta$, and $K$ on $n$.  We will construct a hypothesis test for $H_0:\mathcal M(X)$ versus $H_1:\mathcal{M}(Y)$. Note that since $\mathcal{M}$ is $\ep$-DP for groups of size 1, it is $K\ep$-DP for groups of size $K$ \citep[Theorem 2.2]{Dwork2014}.

Define the sets $S_X$ and $S_Y$ as follows:
\[
\begin{aligned}
S_X &= \{\mathcal{P} \mid \db (\mathcal{P},\mathcal{P}(X)) < K \Delta/2\}\\
S_Y &= \{\mathcal{P} \mid \db (\mathcal{P},\mathcal{P}(Y)) < K \Delta/2\}\\
\end{aligned}
\]
and define our test to be $\phi(\mathcal{M}(\cdot)) = I(\mathcal{M}(\cdot) \in S_Y)$, which is the indicator function on the event $\mathcal{M}(\cdot) \in S_Y$. Then 
\[
\begin{aligned}
\text{Type I error }&= \mathbb{P}(\mathcal{M}(X) \in  S_Y) \leq \mathbb{P}( \mathcal{M}(X) \not\in S_X)\\
\text{Type II error} &= \mathbb{P} (\mathcal{M}(Y) \not\in S_Y).
\end{aligned}
\]
As a result of Lemma \ref{lemma: proof of DTM lower 2}, we have that 
\[
\mathbb{P}( \mathcal{M}(X) \not\in S_X) + \mathbb{P}( \mathcal{M}(X') \not\in S_Y) \geq \frac{2}{1+e^{k\ep}}\geq \frac{2}{1+e},
\]
since $k\ep\leq 1$, which implies that either 
\[
\mathbb{P}(\db (\mathcal{M}(X), \mathcal{P}(X)) \geq K\Delta/2) \geq \frac{1}{1+e},
\]
or 
\[
\mathbb{P}( \db (\mathcal{M}(Y),\mathcal{P}(Y)) \geq K \Delta/2) \geq \frac{1}{1+e}.
\]
This rules out the possibility that both are $o_p(K\Delta)\leq o_p(\Delta/\ep)$.
\end{proof}

\begin{proof}[Proof of Theorem \ref{thm: CS error lower}]
According to Lemma \ref{lemma: proof of DTM lower 1}, it is guaranteed that the $0$th persistence diagram of the $L^1$-DTM filtration $\PDTMone$ satisfies the conditions stated in Lemma \ref{lemma: proof of DTM lower 3} with $\Delta_n = \frac{C}{mn}$ for some constant $C$ independent of $n$. For brevity, set $\mathcal{P} = \PDTMone$. Then, for a sequence $(X_n, Y_n)$ of data sets satisfying the condition, Lemma \ref{lemma: proof of DTM lower 3} tells us that either 
\[
\mathbb{P}( \db (\mathcal{M}(X_n), \mathcal{P}(X_n)) \geq CK/n) \geq \frac{1}{1+e}
\]
or 
\[
\mathbb{P}( \db (\mathcal{M}(Y_n), \mathcal{P}(Y_n)) \geq CK/n) \geq \frac{1}{1+e}
\]
holds, which rules out the possibility that they are $o_p(K/n)\leq o_p(1/(n\ep))$. This completes the proof.
\end{proof}

\section{Supplements of the Simulation and the Real Data Analysis}
\label{appendix: implementation}

\subsection{More Detailed Description of Our Algorithm}
\label{appendix: supplement-alg}

Here, the algorithm of our privacy mechanism, which is introduced in Section \ref{sec: ExpMech}, is explained in detail. For a given data set $D$, Let $S$ denote the maximum value of the $L^1$-DTM function on the data set and $M$ a specified positive integer. 

To get an initial diagram $\mathcal{P}_{\mathrm{DP}}^{(0)} = \big(\mathcal{P}_{0, \mathrm{DP}}^{(0)}, \mathcal{P}_{1, \mathrm{DP}}^{(0)}\big)$, generate independently and identically distributed sample $x_1, \dots, x_M, z_1, \dots, z_M$ from the uniform distribution on the closed interval $[0, S]$, i.e.,  $x_1, \dots, x_n, z_1, \dots, z_M \overset{i.i.d.}{\sim} \mathrm{Unif}[0, S]$, symbolically. For each $i \in \{1, \dots, M\}$, set
\[
y_i = x_i + (1 - x_i) z_i;
\]
the diagram $\mathcal{P}_{0, \mathrm{DP}}^{(0)}$ is constructed as follows:
\[
\mathcal{P}_{0, \mathrm{DP}}^{(0)} = \{(x_1, y_1), \dots, (x_M, y_M)\}.
\]
The other initial diagram $\mathcal{P}_{1, \mathrm{DP}}^{(0)}$ is generated in the same way independent of $\mathcal{P}_{0, \mathrm{DP}}^{(0)}$. Notice that each initial diagram consists of $M$ points uniformly distributed on the upper-left triangle $\{(x, y) \: : \: x, y \in [0, S] \text{ and } y \geqslant x\}$.

For generated $\mathcal{P}_{\mathrm{DP}}^{(t)} = (\mathcal{P}_{0, \mathrm{DP}}^{(t)}, \mathcal{P}_{1, \mathrm{DP}}^{(t)}$), the next candidate $\mathcal{P}_{\mathrm{DP}}^{(t+1)}$ by adding Gaussian noise to each element in each diagram in $t$th step. To be precise, write 
\[
\mathcal{P}_{0, \mathrm{DP}}^{(t)} = \{(x_1^{(t)}, y_1^{(t)}), \dots, (x_M^{(t)}, y_M^{(t)}) \}.
\]
Generate i.i.d. sample $Z_1^{(t)}, \dots, Z_M^{(t)}$ from the $2$-dimensional Gaussian distribution with mean $(0, 0)$ and covariance matrix $\sigma^2 I_2$, where $\sigma$ is a pre-specified positive number and $I_2$ is the $2$ by $2$ identity matrix. Set
\[
\mathcal{P}_0^{'} = \{(x_1^{(t)}, y_1^{(t)}) + Z_1^{(t)}, \dots, (x_M^{(t)}, y_M^{(t)}) + Z_M^{(t)} \}
\]
$\mathcal{P}_1^{'}$ is constructed in the same with independent of $\mathcal{P}_0^{'}$, and set $\mathcal{P}^{'} = (\mathcal{P}_0^{'}, \mathcal{P}_1^{'})$. Then, calculate the accept/reject probability in the Metropolis-Hastings sampler $p$ defined in terms of the bottleneck distance:
\[
p = \min \bigg\{0, - \frac{\epsilon}{2 \Delta} \big( u_D(\mathcal{P}_{\mathrm{DP}}^{(t)}) - u_D(\mathcal{P}^{'}) \big) \bigg\},
\]
where 
\[
\Delta = \frac{2 \sqrt{2}}{mn}.
\]
Generate $U \sim \mathrm{Unif}(0, 1)$. If $\log U \leq p$, take $\mathcal{P}_{\mathrm{DP}}^{(t+1)} = \mathcal{P}^{'}$; otherwise, take $\mathcal{P}_{\mathrm{DP}}^{(t+1)} = \mathcal{P}_{\mathrm{DP}}^{(t)}$. This procedure is carried out repeatedly again, and the $\mathcal{P}_{\mathrm{DP}}^{(t)}$ at the final iteration is proposed as a privatized persistence diagram. The whole procedure is summarized in Algorithm \ref{alg: Our Mech-detailed}.

\begin{algorithm}[htbp]
   \caption{MCMC implementation of the exponential mechanism} 
   \label{alg: Our Mech-detailed}
\begin{algorithmic}
   \STATE {\bfseries Input:}  $\mathcal{P}_{0}(D),\mathcal{P}_{1}(D)$, and a positive integer $M$
   \STATE {\bfseries Initialization:} $\mathcal{P}_{0, \mathrm{DP}}^{(0)},\mathcal{P}_{1, \mathrm{DP}}^{(0)} \sim \mathrm{Unif}(\mathsf{Pers}_{M})$
   \FOR{$i=1,2,\dots,$}
   \STATE $\mathcal{P}_0^{'}=\mathcal{P}_{0, \mathrm{DP}}^{(t-1)}+{\rm N}(0,\sigma^{2} I_2)$, $\mathcal{P}_1^{'} = \mathcal{P}_{1, \mathrm{DP}}^{(t-1)}+{\rm N}(0,\sigma^{2} I_2)$
   \STATE $\mathcal{P}^{'} = (\mathcal{P}_0', \mathcal{P}_1')$
   \STATE $p= \min \big\{0, -\frac{\epsilon}{2\Delta}\big( u_D(\mathcal{P}^{(t-1)})-u_D(\mathcal{P}^{'})\big) \big\}$
   \STATE $U \sim {\rm Unif}(0,1)$
  \IF{$\log U \leq p$}
    \STATE $\mathcal{P}_{0, \mathrm{DP}}^{(t)} = \mathcal{P}_0',\mathcal{P}_{1, \mathrm{DP}}^{(t)} = \mathcal{P}_1'$
  \ELSE
    \STATE $\mathcal{P}_{0, \mathrm{DP}}^{(t)} = \mathcal{P}_{0, \mathrm{DP}}^{(t-1)},\mathcal{P}_{1, \mathrm{DP}}^{(t)} = \mathcal{P}_{1, \mathrm{DP}}^{(t-1)}$
  \ENDIF
   \ENDFOR
\end{algorithmic}
\end{algorithm}

\subsection{Additional Results in the Real Data Analysis}

The following illustration, Figure \ref{fig: real data 3}, depicts the accuracy of privatized persistence diagrams for Walker $A$ and $B$. The procedure of implementing the mechanism is the same with that for Walker C described in Section \ref{sec: real data}. One can see that we also obtain quite accurate privatized diagrams in this cases as well. 

\begin{figure}
     \centering
     \begin{subfigure}[b]{1\textwidth} 
         \centering
         \includegraphics[width=\textwidth]{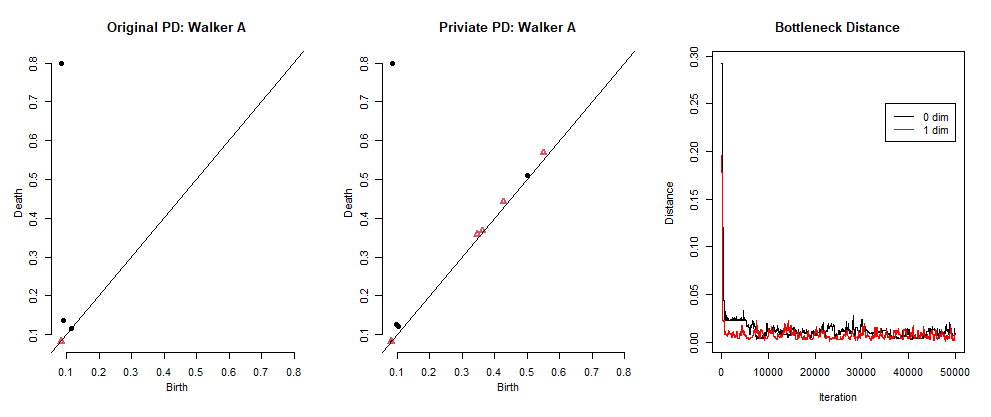}
         \caption{Walker A }
     \end{subfigure}
     \hfill
     \begin{subfigure}[b]{1\textwidth}
         \centering
         \includegraphics[width=\textwidth]{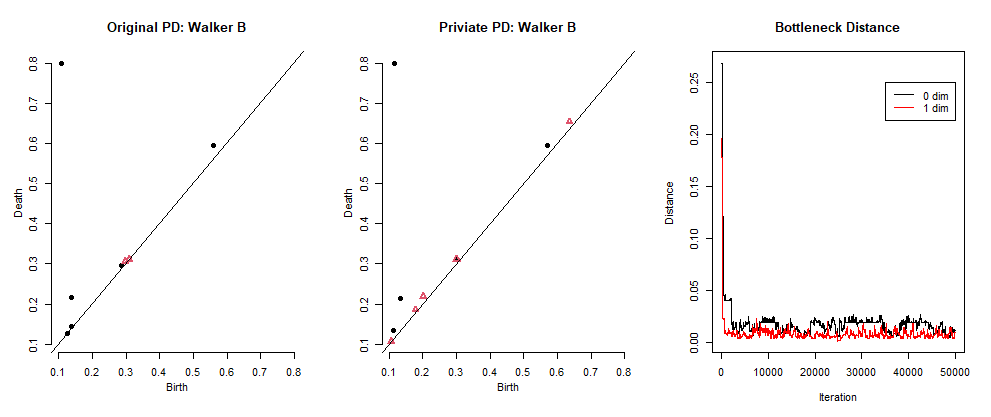}
           \caption{Walker B}
     \end{subfigure}
     \caption{
Figure 8-(a) presents the $L^1$-DTM persistence diagram of the data of Walker A and its privatized diagram. Also, the change of the bottleneck distance between the true and privatized diagram over the MCMC iterations is depicted. At the final iteration, we have $\db (P_{0}(D),P_{0, \mathrm{DP}})=0.01$ and $\db (P_{1}(D),P_{1,\mathrm{DP}})=0.009$. Figure 8-(b) presents the same kind of information about Walker B. Here, at the final iteration, we obtain $\db (P_0(D),P_{0, \mathrm{DP}})=0.011$ and $\db (P_1(D),P_{1, \mathrm{DP}})=0.009$.}
\label{fig: real data 3}
\end{figure}

\end{document}